\documentclass[runningheads]{llncs}

\usepackage{eccv}

\usepackage{eccvabbrv}

\usepackage{graphicx}
\usepackage{booktabs}

\usepackage{multirow}
\usepackage{acro}
\usepackage{xspace}
\usepackage{mathtools}
\usepackage{makecell}
\usepackage{float}
\usepackage{bbm}
\usepackage{algorithm}
\usepackage{algpseudocode}
\usepackage{pgfplots}
\pgfplotsset{compat=1.14}
\usepackage{footnote}
\usepackage[export]{adjustbox}

\usepackage{multibib}
\newcites{supp}{References}

\newcommand{\myparagraph}[1]{\vspace{0.0pt}\noindent{\bf #1}}

\newcommand{\yolox}{YOLOX\xspace}
\DeclareMathOperator*{\argmax}{argmax} %

\definecolor{codeblue}{rgb}{0.25,0.5,0.5}

\algnewcommand\algorithmicinput{\textbf{Input:}}
\algnewcommand\algorithmicoutput{\textbf{Output:}}
\algnewcommand\INPUT{\item[\algorithmicinput]}
\algnewcommand\OUTPUT{\item[\algorithmicoutput]}
\algnewcommand\RETURN{\item[\algorithmicreturn]}
\algnewcommand\algorithmicinputt{\textbf{Track:}}
\algnewcommand\TRACK{\item[\algorithmicinputt]}
\newcommand{\LineComment}[1]{\hfill \textcolor{codeblue}{\# #1}}
\newcommand{\AlgoComment}[1]{\textcolor{codeblue}{\# #1}}

\makeatletter
\newcommand{\greencell}{%
  \@ifstar{\cellcolor[HTML]{E0FCE0}}{}}
\makeatother

\usepackage{soul}
\definecolor{mygreen}{HTML}{E0FCE0}

\usepackage{amssymb}%
\usepackage{pifont}%
\newcommand{\cmark}{\ding{51}}%
\newcommand{\xmark}{\ding{55}}%

\colorlet{colorFstTransp}{Green!25}       %
\colorlet{colorSndTransp}{SpringGreen!45} %
\colorlet{colorTrdTransp}{Yellow!30}      %
\colorlet{colorFst}{Green!99}       %
\colorlet{colorSnd}{SpringGreen!99} %
\colorlet{colorTrd}{Yellow!99}      %
\colorlet{colorGray}{Gray!30}      %

\newcommand{\gr}{\color{gray}}      %
\definecolor{lightgray}{gray}{0.9}

\DeclareAcronym{ssl}{
  short = SSL ,
  long = self-supervised learning
}
\DeclareAcronym{toag}{
  short = TOAG ,
  long = temporal object appearance graph
}
\DeclareAcronym{mot}{
  short = MOT ,
  long = multiple object tracking
}
\DeclareAcronym{reid}{
  short = Re-ID ,
  long = re-identification
}
\DeclareAcronym{nms}{
  short = NMS ,
  long = non maximum suppression
}
\DeclareAcronym{rpn}{
  short = RPN ,
  long = region proposal network
}
\DeclareAcronym{iou}{
  short = IoU ,
  long = Intersection over Union 
}
\DeclareAcronym{roi}{
  short = RoI ,
  long = region of interest ,
  long-plural-form = regions of interest
}
\DeclareAcronym{fpn}{
  short = FPN ,
  long = Feature Pyramid Network ,
}
\DeclareAcronym{deta}{
  short = DetA ,
  long = detection accuracy ,
}
\DeclareAcronym{assa}{
  short = AssA ,
  long = association accuracy ,
}

\usepackage[accsupp]{axessibility}  %

\usepackage{hyperref}

\usepackage{orcidlink}

\usepackage[capitalize]{cleveref}
\crefname{section}{Sec.}{Secs.}
\Crefname{section}{Section}{Sections}
\Crefname{table}{Table}{Tables}
\crefname{table}{Tab.}{Tabs.}
\Crefname{algorithm}{Algorithm}{Algorithms}
\crefname{algorithm}{Alg.}{Algs.}

\newread\imgstream
\immediate\openin\imgstream=imagedata.in
\makeatletter
\def\new@kvginclip#1{}
\def\new@kvgintrim#1{}
\let\old@kvginclip\KV@Gin@clip
\let\old@kvgintrim\KV@Gin@trim
\let\oldincludegraphics\includegraphics
\renewcommand{\includegraphics}[2][]{%
  \immediate\read\imgstream to \src
  \immediate\read\imgstream to \removecrop
  \ifnum\removecrop=1
      \let\KV@Gin@clip\new@kvginclip
      \let\KV@Gin@trim\new@kvgintrim
  \fi
  \oldincludegraphics[#1]{\src}%
  \let\KV@Gin@clip\old@kvginclip
  \let\KV@Gin@trim\old@kvgintrim}
\makeatother

\begin{document}

\title{Walker: Self-supervised Multiple Object Tracking by Walking on Temporal Appearance Graphs}

\titlerunning{Walker: Self-supervised Multiple Object Tracking}
\author{Mattia Segu\inst{1,2} \and Luigi Piccinelli\inst{1} \and Siyuan Li\inst{1} \and \\ Luc Van Gool\inst{1,3} \and Fisher Yu\inst{1} \and Bernt Schiele\inst{2}}
\institute{ETH Zurich, Switzerland \and
Max Planck Institute for Informatics, Saarland Informatics Campus, Germany \and INSAIT, Bulgaria \\
\url{https://github.com/mattiasegu/walker}
}

\authorrunning{M.~Segu et al.}

\maketitle
\begin{abstract}
   The supervision of state-of-the-art multiple object tracking (MOT) methods requires enormous annotation efforts to provide bounding boxes for all frames of all videos, and instance IDs to associate them through time. 
   To this end, we introduce Walker, the first self-supervised tracker that learns from videos with sparse bounding box annotations, and no tracking labels. 
   First, we design a quasi-dense temporal object appearance graph, and propose a novel multi-positive contrastive objective to optimize random walks on the graph and learn instance similarities.
   Then, we introduce an algorithm to enforce mutually-exclusive connective properties across instances in the graph, optimizing the learned topology for MOT.
   At inference time, we propose to associate detected instances to tracklets based on the max-likelihood transition state under motion-constrained bi-directional walks.
   Walker is the first self-supervised tracker to achieve competitive performance on MOT17, DanceTrack, and BDD100K.
   Remarkably, our proposal outperforms the previous self-supervised trackers even when drastically reducing the annotation requirements by up to 400x. 
\end{abstract}

\section{Introduction}
\Ac{mot} represents a cornerstone of modern perception systems for challenging computer vision applications, such as autonomous driving~\cite{ess2010object}, video surveillance~\cite{elhoseny2020multi}, and augmented reality~\cite{park2008multiple}.
Following the tracking-by-detection paradigm, multiple object trackers detect objects in all frames (object detection) while associating them through time (data association) to obtain tracklets.
Modern trackers~\cite{du2023strongsort,aharon2022bot,wang2022smiletrack} achieve state-of-the-art performance by combining motion heuristics~\cite{bewley2016simple,wojke2017simple,zhang2022bytetrack} with learned appearance descriptors~\cite{wojke2017simple,pang2021quasi,zhang2021fairmot} for data association.
As such, the supervision of multiple object trackers requires annotating detection labels - \ie bounding boxes - 
in every frame for all the objects of the categories of interest, and tracking labels as instance IDs to associate objects through time (\cref{fig:teaser}, top).
Thus, the annotation cost of \ac{mot} datasets~\cite{dendorfer2021motchallenge,sun2022dancetrack,yu2020bdd100k,sun2022shift,sun2020scalability} is linear in the number of frames, and labeling large video datasets can be prohibitive.

\begin{table}[t]
  \centering
  \captionsetup{type=figure}
    \resizebox{\textwidth}{!}{
    \fontsize{5}{7}\selectfont
    \setlength{\tabcolsep}{0.5pt}
    \begin{tabular}{c|ccccccccc}    
      & Det.  & IDs  &  &  &  & &   \\ \cline{1-3}
    \raisebox{+0.0\normalbaselineskip}[0pt][0pt]{\rotatebox[origin=c]{90}{\makecell{Sup. \\ MOT}}}  & \raisebox{+0.0\normalbaselineskip}[0pt][0pt]{Dense}  & \raisebox{+0.0\normalbaselineskip}[0pt][0pt]{\cmark} & \includegraphics[width=0.14\textwidth,valign=m,trim={16.0cm 5.0cm 1.0cm 4.0cm},clip]{figures/teaser/dense-supervised/000010.jpg} & \includegraphics[width=0.14\textwidth,valign=m,trim={16.0cm 5.0cm 1.0cm 4.0cm},clip]{figures/teaser/dense-supervised/000011.jpg} & \includegraphics[width=0.14\textwidth,valign=m,trim={16.0cm 5.0cm 1.0cm 4.0cm},clip]{figures/teaser/dense-supervised/000012.jpg} & \includegraphics[width=0.14\textwidth,valign=m,trim={16.0cm 5.0cm 1.0cm 4.0cm},clip]{figures/teaser/dense-supervised/000013.jpg} & \includegraphics[width=0.14\textwidth,valign=m,trim={16.0cm 5.0cm 1.0cm 4.0cm},clip]{figures/teaser/dense-supervised/000014.jpg} & \includegraphics[width=0.14\textwidth,valign=m,trim={16.0cm 5.0cm 1.0cm 4.0cm},clip]{figures/teaser/dense-supervised/000015.jpg} & \includegraphics[width=0.14\textwidth,valign=m,trim={16.0cm 5.0cm 1.0cm 4.0cm},clip]{figures/teaser/dense-supervised/000016.jpg} \\ \cline{1-3}
    \raisebox{+0.0\normalbaselineskip}[0pt][0pt]{\rotatebox[origin=c]{90}{\makecell{Self-sup. \\ Re-ID}}}  &  \raisebox{+0.0\normalbaselineskip}[0pt][0pt]{Dense}  & \raisebox{+0.0\normalbaselineskip}[0pt][0pt]{\xmark} & \includegraphics[width=0.14\textwidth,valign=m,trim={16.0cm 5.0cm 1.0cm 4.0cm},clip]{figures/teaser/dense-self_supervised/000010.jpg} & \includegraphics[width=0.14\textwidth,valign=m,trim={16.0cm 5.0cm 1.0cm 4.0cm},clip]{figures/teaser/dense-self_supervised/000011.jpg} & \includegraphics[width=0.14\textwidth,valign=m,trim={16.0cm 5.0cm 1.0cm 4.0cm},clip]{figures/teaser/dense-self_supervised/000012.jpg} & \includegraphics[width=0.14\textwidth,valign=m,trim={16.0cm 5.0cm 1.0cm 4.0cm},clip]{figures/teaser/dense-self_supervised/000013.jpg} & \includegraphics[width=0.14\textwidth,valign=m,trim={16.0cm 5.0cm 1.0cm 4.0cm},clip]{figures/teaser/dense-self_supervised/000014.jpg} & \includegraphics[width=0.14\textwidth,valign=m,trim={16.0cm 5.0cm 1.0cm 4.0cm},clip]{figures/teaser/dense-self_supervised/000015.jpg} & \includegraphics[width=0.14\textwidth,valign=m,trim={16.0cm 5.0cm 1.0cm 4.0cm},clip]{figures/teaser/dense-self_supervised/000016.jpg} \\ \cline{1-3}
    \raisebox{+0.0\normalbaselineskip}[0pt][0pt]{\rotatebox[origin=c]{90}{\makecell{Ours}}}  &  \raisebox{+0.0\normalbaselineskip}[0pt][0pt]{Sparse}  & \raisebox{+0.0\normalbaselineskip}[0pt][0pt]{\xmark} & \includegraphics[width=0.14\textwidth,valign=m,trim={16.0cm 5.0cm 1.0cm 4.0cm},clip]{figures/teaser/sparse-self_supervised/000010.jpg} & \includegraphics[width=0.14\textwidth,valign=m,trim={16.0cm 5.0cm 1.0cm 4.0cm},clip]{figures/teaser/sparse-self_supervised/000011.jpg} & \includegraphics[width=0.14\textwidth,valign=m,trim={16.0cm 5.0cm 1.0cm 4.0cm},clip]{figures/teaser/sparse-self_supervised/000012.jpg} & \includegraphics[width=0.14\textwidth,valign=m,trim={16.0cm 5.0cm 1.0cm 4.0cm},clip]{figures/teaser/sparse-self_supervised/000013.jpg} & \includegraphics[width=0.14\textwidth,valign=m,trim={16.0cm 5.0cm 1.0cm 4.0cm},clip]{figures/teaser/sparse-self_supervised/000014.jpg} & \includegraphics[width=0.14\textwidth,valign=m,trim={16.0cm 5.0cm 1.0cm 4.0cm},clip]{figures/teaser/sparse-self_supervised/000015.jpg} & \includegraphics[width=0.14\textwidth,valign=m,trim={16.0cm 5.0cm 1.0cm 4.0cm},clip]{figures/teaser/sparse-self_supervised/000016.jpg} \\ \cline{1-3}
    \end{tabular}
    }
      
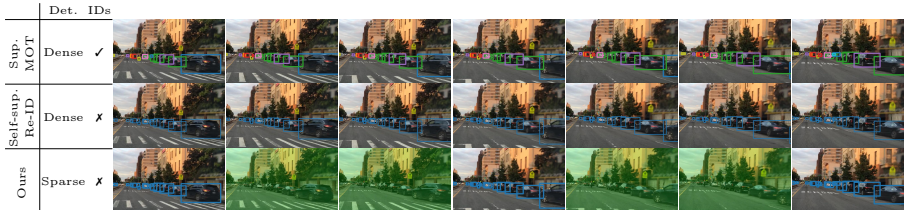
\captionof{figure}{
      Supervised MOT requires dense tracking labels~(top), \ie dense detection annotations at each frame and instance labels (shown by coloring boxes by instance ID) across frames. Self-supervised Re-ID assumes dense detection labels and no instance labels~(middle). We explore self-supervised MOT in a more practical sparsely-annotated setting~(bottom),
      with sparse detection annotations every $k$ frames (here $k=3$ for illustration purpose) and no instance labels. 
      Fully-unlabeled frames in green.
      } \label{fig:teaser}
\end{table}%

Self-supervised \ac{mot} - the problem of learning to track in the absence of the instance labels - represents an appealing solution to alleviate the enormous annotation cost.
Nevertheless, 
the most common self-supervised \ac{mot} solutions~\cite{fischer2022qdtrack,segu2023darth,li2023ovtrack,heigold2023video,zhang2021fairmot} only rely on image-level self-supervision. 
By not leveraging the privileged temporal information of video streams, these approaches cannot learn appearance descriptors robust to view changes, and fail to close the gap with supervised \ac{mot}.
Analogously, orthogonal research on self-supervised \ac{reid}~\cite{kim2023ssl,yang2020remots,gan2021self,bastani2021self} traditionally assumes high-quality dense detection annotations in videos (\cref{fig:teaser}, middle), hindering label-efficiency.
We argue that video-level self-supervision should both enable discarding instance ID annotations and greatly sparsify the redundant detection labels (\cref{fig:teaser}, bottom). 

To this end, we introduce Walker, the first self-supervised multiple object tracker to learn from videos with sparse bounding box annotations and no tracking labels.
Walker is a joint detection and tracking model composed of a detector and a cascaded embedding head.
Inspired by~\cite{jabri2020space}, we design a \ac{toag} (\cref{ssec:method_appearance_graph}) that connects object-level \acp{roi} on a pair of key/reference frames.
During training, we propose to self-supervise appearance representations by walking on \acp{toag}.
First, we introduce a novel multi-positive contrastive formulation to optimize cyclic random walks on the graph and learn instance similarities (\cref{ssec:method_cycle}).
Then, we propose an algorithm to identify pseudo-matches between key and reference clusters of detections as the max-likelihood transition states over the cycle walks connecting them. Given such assignments, we enforce a mutually-exclusive graph connectivity across instances as required for \ac{mot} (\cref{ssec:method_forward}). 
At inference time, we propose a more refined appearance similarity metric - namely the \textit{biwalk} - to associate detections to tracklets by finding the max-likelihood transition state under the motion-constrained cycle walks connecting them (\cref{ssec:method_walker_matching}).

Moreover, we investigate the efficacy of self-supervised \ac{mot} by sparsifying the dense detection annotations requirement, \ie providing ground-truth bounding boxes only every $k$ frames in a video (\cref{fig:teaser}, bottom). 
By relying on our video-level self-supervision, we find that Walker effectively leverages fully-unlabeled frames to learn superior appearance representations, significantly outperforming the frame-level self-supervised \ac{mot} state of the art~\cite{fischer2022qdtrack} even when training with up to 400x less annotated frames (\cref{fig:annotation_sparsity_assa}).
Finally, experimental results on MOT17~\cite{dendorfer2021motchallenge}, DanceTrack~\cite{sun2022dancetrack}, and BDD100K~\cite{yu2020bdd100k} highlight that Walker is the first self-supervised tracker competitive with state-of-the-art supervised ones.

We summarize our contributions: 
(i) we introduce Walker, the first self-supervised multi-object tracker to learn  appearance from sparsely annotated videos and no tracking labels;
(ii) we propose a novel video-level self-supervision formulation that learns instance similarities with multi-positive and mutually-exclusive contrastive random walks on temporal object appearance graphs;
(iii) Walker is the first self-supervised tracker competitive with state-of-the-art supervised \ac{mot}, while greatly reducing the annotation requirements.

\section{Related Work}
\myparagraph{Multiple Object Tracking.}
Most \ac{mot} approaches rely on the tracking-by-detection paradigm, \ie objects are detected in each frame while data association matches the detected instances across frames.
\textit{Motion-based} heuristics have long been used to associate objects through time~\cite{reid1979algorithm,bewley2016simple,zhang2022bytetrack}. SORT~\cite{bewley2016simple} first predicts the future location of the tracklets with a Kalman filter~\cite{kalman1960new} and then matches predicted to detected boxes using \ac{iou} as a measure of spatial similarity.
ByteTrack~\cite{zhang2022bytetrack} proposes a two-stage matching strategy to properly utilize low-score detections.
However, motion-based trackers struggle under occlusions, low frame rates, and complex camera and objects motion~\cite{fischer2022qdtrack}.
DeepSORT~\cite{wojke2017simple}, StrongSORT~\cite{du2023strongsort} and BoT-SORT~\cite{aharon2022bot} extend SORT with a stand-alone \ac{reid} module for occlusion-handling, and train it on an external pedestrian re-identification dataset~\cite{zheng2016mars} to extract \textit{appearance-based} representations.
However, their parallel \ac{reid} module undermines efficiency and is trained on external data.
Recent \textit{joint detection and tracking} models~\cite{lu2020retinatrack,wang2020towards,zhang2021fairmot,pang2021quasi,fischer2022qdtrack} extend the detector's feature extractor with an embedding head for efficient appearance extraction.
QDTrack's~\cite{pang2021quasi,fischer2022qdtrack} quasi-dense contrastive formulation proved an effective in-domain appearance-learning scheme~\cite{fischer2022qdtrack}.
Queries in query-based trackers~\cite{sun2020transtrack,meinhardt2022trackformer,zeng2022motr,segu2024samba} are also implicit appearance representations.
While appearance complements motion-based trackers, it comes with a high annotation cost. Training appearance extractors in-domain necessitates tracking datasets to provide detection and instance ID annotations for all frames in a video (\cref{fig:teaser}, top).
Our work overcomes these limitations by proposing a self-supervised appearance-learning algorithm that eliminates the need for instance-association labels, and allows for sparser detection annotations (\cref{fig:teaser}, bottom).

\myparagraph{Self-supervised Re-ID.}
Self-supervised \ac{reid}~\cite{kim2023ssl,yang2020remots,gan2021self,bastani2021self} is the problem of learning instance representations given ground-truth detections (\cref{fig:teaser}, middle). 
\cite{kim2023ssl,collicottself,huang2023multi} learn \ac{reid} with image-level self-supervision via pre-text tasks - \eg image rotation, puzzle solving, reconstruction, MoCo-v2~\cite{chen2020improved}, BYOL~\cite{grill2020bootstrap}.
Other techniques learn \ac{reid} directly on in-domain videos by means of weak clustering labels obtained with tracking algorithms~\cite{ho2020two,karthik2020simple,yang2020remots}, or cycle consistency~\cite{gan2021self,bastani2021self} on ground-truth bounding boxes.
By assuming availability of ground-truth detections, such approaches are not designed for joint detection and tracking.

\myparagraph{Self-supervised Multiple Object Tracking.}
Despite the recent advances in self-supervised correspondence learning in videos~\cite{jabri2020space,wang2019learning,gupta2023siamese}, frame-level self-supervision is the standard in \ac{mot}.
QDTrack-S(tatic)~\cite{fischer2022qdtrack} generates two views of the same frame with data augmentation and optimizes a contrastive loss on the embeddings of different instances. 
Due to its simplicity, this paradigm has been adopted in test-time adaptive~\cite{segu2023darth}, open-vocabulary~\cite{li2023ovtrack,yang2019video,li2024slack} and foundational tracking~\cite{li2024matching}.
However, \ac{mot} requires associating instances through time, and data augmentation cannot mimic the occlusions, pose changes, and distortions of real videos. 
By walking on temporal appearance graphs, our method benefits from the video information to learn superior appearance representations.

\section{Walker}
We introduce our novel self-supervised tracker, Walker.
We report architectural details in \cref{ssec:method_architecture}, and define our proposed quasi-dense temporal object appearance graph (\cref{ssec:method_appearance_graph}). 
We then introduce our techniques to train the \ac{toag} and learn instance descriptors from unlabeled videos: a novel multi-positive contrastive objective to optimize random walks on the appearance graph - after which Walker is named -    (\cref{ssec:method_cycle}); our approach to identify pseudo-assignments and optimize mutually-exclusive connectivity on the graph (\cref{ssec:method_forward}).
Finally, we detail Walker's data association scheme and introduce our biwalk similarity metric (\cref{ssec:method_walker_matching}) to track objects based on the learned appearance graph.

\subsection{Architecture} \label{ssec:method_architecture}
Our tracker can be coupled with any two-stage and one-stage detector for end-to-end training. 
The object detector is composed of a feature extractor with a \ac{fpn} to extract multi-scale feature maps and a bounding box head.
An additional embedding head extracts deeper appearance representations for each \ac{roi} after \ac{roi}Align~\cite{he2017mask}.
For two-stage detectors, we treat the region proposals as \acp{roi}; for one-stage detectors, the detections after \ac{nms}.
Following state-of-the-art appearance-~\cite{fischer2022qdtrack} and motion-based~\cite{zhang2022bytetrack} trackers, we choose \yolox as \textit{detector}, while our \textit{embedding head} is a 4conv-1fc head with group normalization~\cite{wu2018group} to extract 256-dimensional features as in QDTrack~\cite{fischer2022qdtrack}.

\subsection{Temporal Object Appearance Graphs} \label{ssec:method_appearance_graph}
We introduce a self-supervised formulation to learn instance similarities by walking on quasi-dense temporal object appearance graphs (\acp{toag}).
Inspired by the contrastive random walk for self-supervised pixel-level correspondences~\cite{jabri2020space}, we represent each video as a quasi-dense~\cite{fischer2022qdtrack} directed appearance graph $\mathcal{G}$ where nodes are the quasi-dense \acp{roi}, and weighted edges connect nodes in neighboring frames.
Unlike~\cite{jabri2020space}, our work redefines the appearance graph to walk on quasi-dense object regions, introduces a new multi-positive self-supervised objective (\cref{ssec:method_cycle}), and enforces mutually-exclusive connective properties across instances (\cref{ssec:method_forward}) to make the learned topology optimal for \ac{mot}. 

\myparagraph{Nodes Definition.} 
We define the graph nodes for an image $I_t$ at time $t$ as its \acp{roi}, and describe them by their appearance embeddings.
Given the set of high-confidence detections ${\mathcal{D}_t^{\text{high}} = \{d_t^i \; | \; \text{conf}(d_t^i) \ge \beta_{\text{obj}} \! = \! 0.3 \}}$
predicted by the detector on $I_t$, or the set of ground-truth boxes ${\hat{\mathcal{D}}_t = \{d_t^i\}}$, we define a \ac{roi} as positive to a detection $d^i_t$ if their \ac{iou} is higher than $\alpha_1 \! = \! 0.7$, negative if lower than $\alpha_2 \! = \! 0.3$. 
We use \ac{roi} Align~\cite{he2017mask} to pool feature maps at different levels in the \ac{fpn}~\cite{lin2017feature} according to the \ac{roi} scales.
For each frame $I_t$, we select 128 positive \acp{roi} $\mathbf{Q}_t^+$ and 128 negative $\mathbf{Q}_t^-$ ones, and describe the nodes $\mathbf{Q}_t = \mathbf{Q}_t^+ \cup \mathbf{Q}_t^-$ by the corresponding embeddings matrix $Q_t = [Q_t^+, Q_t^-]$ obtained by applying the embedding head on the pooled \ac{roi} features.
In contrast to~\cite{jabri2020space}, our nodes are object-centric \acp{roi} instead of patches to learn instance-specific representations.

\myparagraph{Cluster Definition.}
Given the quasi-dense nature of our \ac{toag}, multiple nodes can represent different views of the same object.
We define the cluster {\small${\mathcal{C}_t^i = \mathcal{C}_t(\textbf{q}_t^i) = \{\textbf{q}_t^j \in \textbf{Q}_t \; | \; \text{\ac{iou}}(\textbf{q}_t^j, \textbf{q}_t^i) \ge \alpha_1 = 0.7\}}$} as the set of nodes sufficiently overlapping with the $i$-th node $\textbf{q}_t^i$ in $I_t$.
Given the high overlap, all \acp{roi} in a cluster $\mathcal{C}_t(\textbf{q}_t^i)$ typically represent the same instance, \ie a specific pedestrian. 

\myparagraph{Edges Definition.}
We define the edges $A_t^{t'}(i,j)$ connecting the nodes $\mathbf{q}_t^i$ and $\mathbf{q}_{t'}^j$ across $I_t$ and $I_{t'}$ by the cosine similarities 
{\small$c(q_t^i, q_{t'}^j) = (q_t^i \cdot q_{t'}^j)/(||q_t^i|| ||q_{t'}^j||)$}
between the nodes' embeddings $q_t^i$ and $q_{t'}^j$, transformed into non-negative affinities by a softmax with temperature $\tau$ over edges departing from each node $\mathbf{q}_t^i$ directed to all nodes $\mathbf{q}_{t'}^i \in \textbf{Q}_{t'}$. $A_t^{t'}$ is the local transition matrix from $\textbf{Q}_{t}$ to $\textbf{Q}_{t'}$ on $\mathcal{G}$:
\setlength{\abovedisplayskip}{4.0pt}
\setlength{\belowdisplayskip}{4.0pt}
\small\begin{equation}
A_t^{t'}(i,j) = \texttt{softmax}_i(Q_t Q_{t'}^\top)(i, j) =
             \frac{exp(c(q_t^i, q_{t'}^j)/\tau)}{\sum_{l=1}^{N}exp(c(q_t^i, q_{t'}^l)/\tau)},
\end{equation}\normalsize
Unlike~\cite{jabri2020space} and since our edges represent the instance similarities used for tracking, the optimal topology of $\mathcal{G}$ for \ac{mot}  must present mutually-exclusive connective properties across clusters of nodes - \ie nodes from one instance can only transition to other nodes of the same instance - which we enforce in \cref{ssec:method_forward}. 

\myparagraph{Temporal Appearance Graph Definition.} An appearance graph $\mathcal{G}$ defined by the nodes and edges described above is a spatio-temporal Markov chain whose transition probabilities between its quasi-dense states are given by the non-negative affinity matrix $A_t^{t'}(i,j)=P(X_{t'}=j|X_{t}=i)=p_{X_{t'} | X_{t}}(j|i)$, where $X_t$ is the state of a walker at time $t$ and $P(X_t=i)$ is the probability of being at node $i$ at time $t$.
In \cref{ssec:method_cycle,ssec:method_forward,ssec:method_total_loss} we show how to learn a mutually-exclusive \ac{toag}, and in \cref{ssec:method_walker_matching} how to use it for tracking.

\begin{figure*}[t]
    \centering
    \footnotesize
    \includegraphics[width=0.8\textwidth,trim={0.0cm 0.0cm 5.65cm 9.0cm},clip]{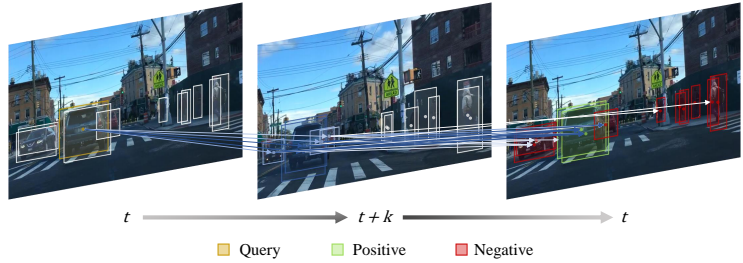}
    \caption{\textbf{Multi-positive Cycle Consistency.} Illustration of the proposed multi-positive cycle consistency on quasi-dense \acp{toag} (\cref{ssec:method_cycle}). We show the cycle walk departing from a given query node (yellow). The multiple positive (negative) nodes are in green (red). For ease of visualization, we only show the high-likelihood transitions. } \label{fig:method_cycle}
\end{figure*}

\subsection{Learning Instance Representations by Walking on Cyclic Object Appearance Graphs} \label{ssec:method_cycle}
In absence of instance ID labels, we propose to self-supervise instance similarities (edges) by optimizing multi-positive contrastive random walks on cyclic \acp{toag}.

\myparagraph{Cycle Walk Definition.}
Given a key image $I_t$ and its bounding box annotations, we randomly sample an unlabeled reference image $I_{t+k}$ from its temporal neighborhood, \ie $k \in [-\hat{k}, \hat{k}]$, with $\hat{k}$ dataset-dependent.
We build a cyclic appearance graph $\mathcal{G}$ (\cref{fig:method_cycle}) as a walk from the positive nodes $\textbf{Q}_t^+$ - likely to represent objects - in the key image $I_t$ to all the nodes $\textbf{Q}_{t+k}$ in the reference image $I_{t+k}$ and back to all 
nodes $\textbf{Q}_t=[\textbf{Q}_t^+,\textbf{Q}_t^-]$ in $I_t$.
The resulting walk $\mathcal{G}: \textbf{Q}_t^+ \rightarrow \textbf{Q}_{t+k} \rightarrow \textbf{Q}_t$ is a Markov chain described by the forward and backward transitions $A_{t^+}^{t+k}$ and ${A}_{t+k}^{t}$, whose chained transition $\bar{A}_{t^+}^{t}$ describes the cycle correspondence as a multi-step walk along the object appearance graph $\mathcal{G}$: 
\setlength{\abovedisplayskip}{4.0pt}
\setlength{\belowdisplayskip}{4.0pt}
\small
\begin{equation}
\bar{A}_{t^+}^{t} = A_{t^+}^{t+k} {A}_{t+k}^{t}  = P_{\mathcal{G}}(X_t | X_{t+k}) P_{\mathcal{G}}(X_{t+k} | X^+_t) = P_{\mathcal{G}}(X_t | X^+_t).
\end{equation}\normalsize

\myparagraph{Multi-positive Cycle Consistency.}
Cycle consistency is satisfied for a node $\textbf{q}_t^{i}$ in $I_t$ if {\small${p^{\mathcal{G}}_{X_t|X_t^+}(i|i) > p^{\mathcal{G}}_{X_t|X_t^+}(j|i) \; \forall \; j \ne i}$}, \ie a cycle walk on $\mathcal{G}$ starting from $\textbf{q}_t^i$ ends on  $\textbf{q}_t^i$ itself.
However, since the above-defined graph is quasi-dense, we can identify multiple positive targets $Y^+_i$ for the walk starting from $\textbf{q}_t^i$ as the cluster $\mathcal{C}_t(\textbf{q}_{t}^i)$ of nodes $\textbf{q}_t^l$ sufficiently overlapping with the starting node $\textbf{q}_t^i$, \ie 
{\small${Y^+_i = \mathcal{C}_t(\textbf{q}_{t}^i) = \{\textbf{q}_t^j \in \textbf{Q}_t \; | \; \text{\ac{iou}}(\textbf{q}_t^j, \textbf{q}_t^i) \ge \alpha_1 = 0.7\}}$}.
All other nodes are considered negative targets to $\textbf{q}_t^i$, \ie 
{\small${Y^-_i = \{\textbf{q}_t^{j} \; | \; \textbf{q}_t^{i} \notin Y^+_i \; \forall \; \textbf{q}_t^j \in \textbf{Q}_t\}}$}. 
\cref{fig:method_cycle} illustrates the positive (green) and negative (red) targets for a cycle walk starting from a query node (yellow).
We consider \textit{multi-positive cycle consistency} satisfied if:
\setlength{\abovedisplayskip}{4.0pt}
\setlength{\belowdisplayskip}{4.0pt}
\small\begin{equation}
p^{\mathcal{G}}_{X_t |X_t^+}(Y_i^+ | i) = \sum_{\textbf{q}_t^{l} \in Y^+_i}{p^{\mathcal{G}}_{X_t | X_t^+}(l|i) > p^{\mathcal{G}}_{X_t|X_t^+}(j|i) \; \forall \; \textbf{q}_t^{j} \notin Y_i^+}.
\end{equation}\normalsize
Meaningful pairwise instance similarities must emerge to solve the cyclic walk on the graph, such that each node walks back to one of its multiple positive targets when a latent correspondence is found in $I_{t+k}$.
In \ac{mot}, a desired latent correspondence in $I_{t+k}$ to a \ac{roi} in $I_{t}$ is a \ac{roi} representing the same instance. 
We introduce a novel \textit{multi-positive contrastive loss} on the cycle probabilities to solve the quasi-dense cycle consistency problem and let latent matches emerge for all starting object nodes $\textbf{q}_t^i \in \textbf{Q}_t^+$, with {\small$\bar{A}_{t^+}^t(i,j) = p^{\mathcal{G}}_{X_t| X_t^+}(j|i)$} probability of closing in $\textbf{q}_t^j$ a cycle on $\mathcal{G}$ that starts from $\textbf{q}_t^i$:
\setlength{\abovedisplayskip}{4.0pt}
\setlength{\belowdisplayskip}{4.0pt}
\small\begin{equation}
\mathcal{L}_{\text{cycle}} = \sum_{\mathbf{q}_t^i \in \mathbf{Q}_t^+} \text{log}( 1 + \sum_{\textbf{q}_t^{l} \in Y^+_i}\sum_{\textbf{q}_t^{j} \in Y^-_i} \text{exp}(\bar{A}_{t^+}^t(i,j)-\bar{A}_{t^+}^t(i,l))).
\end{equation}\normalsize

\begin{figure*}[t]
    \centering
    \footnotesize
    \includegraphics[width=0.8\textwidth,trim={0.0cm 0.0cm 5.65cm 9.0cm},clip]{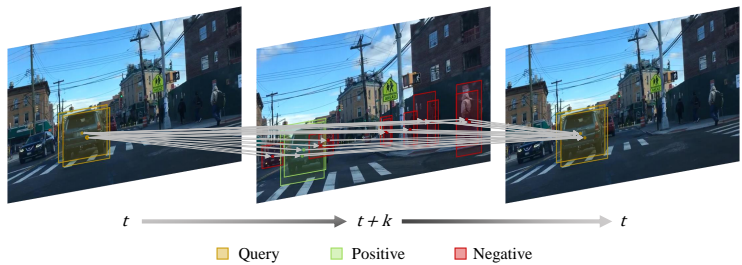}
    \caption{\textbf{Cluster-wise Forward Assignment.} Illustration of the positive (green) and negative (red) forward pseudo-labels for an input query cluster (yellow), deriving from our cluster-wise forward assignment strategy described in \cref{ssec:method_forward}.} \label{fig:method_forward}
\end{figure*}

\subsection{Enforcing Mutually-exclusive Assignments} \label{ssec:method_forward}
For a given starting node $\textbf{q}_t^i$ in $I_t$, enforcing our multi-positive cycle consistency allows the emergence of multiple latent correspondences in the reference frame $I_{t+k}$, \ie multiple nodes $\textbf{q}_{t+k}^j$ with high transition probability ${p^{\mathcal{G}}_{X_{t+k} | X_t, X_t^+}(j|Y_i^+,i)}$ on the cycle walk $\mathcal{G}_i$. However, it is not guaranteed that all such correspondences belong to the same instance.
In \ac{mot}, where the optimal graph topology must exhibit mutually-exclusive connective properties, having multiple instances in $I_{t+k}$ linked to the same instance in $I_t$ is undesirable.
To this end, we propose to (i) identify cluster-wise forward assignments on our cyclic appearance graph (\cref{fig:method_forward}), and (ii) optimize the corresponding transition probabilities to satisfy mutually-exclusive connectivity. Pseudo-code is in the Appendix.

\myparagraph{Cluster-wise Forward Assignment.}
In \cref{app:sec:proof_transition} (Appendix), we prove that the probability of transitioning on a latent node $\textbf{q}_{t+k}^j$ on the reference image $I_{t+k}$ when starting from $\textbf{q}_{t^+}^i$ in $I_t$ and ending on $\textbf{q}_{t}^l$ in $I_t$ along the cycle walk $\mathcal{G}$ is:
\setlength{\abovedisplayskip}{5.0pt}
\setlength{\belowdisplayskip}{4.0pt}
\small\begin{align} \label{eq:transition_probability}
p^{\mathcal{G}}_{X_{t+k}|X_t,X_t^+}(j|l,i)
    &= p^{\mathcal{G}}_{X_t|X_{t+k}}(l|j)p^{\mathcal{G}}{X_{t+k},X_t^+}(j|i) / C  \\
    &= A_{t^+}^{t+k}(i,j)A_{t+k}^t(j,l) / C
\end{align}\normalsize
where {\small$C = \sum_{\textbf{q}_{t+k}^m \in \textbf{q}_{t+k}} p^{\mathcal{G}}_{X_t | X_{t+k}}(l | m) p^{\mathcal{G}}_{X_{t+k} | X_t^+}(m|i)$\normalsize}.

In our quasi-dense setting (\cref{fig:method_forward}), the cluster of nodes $\mathcal{C}_t^i$  around $\textbf{q}_{t^+}^i$ in $I_t$ shares the set of multiple targets $Y_i^+=\mathcal{C}_t^i$ with cardinality $||Y_i^+||$ in $I_t$ for the cycle walk $\mathcal{G}_i$.
For a node  $\textbf{q}_{t^+}^i$ in $I_t$, we can thus refine the probability estimate of traversing a reference node by averaging over all cycles starting from $\mathcal{C}_t^i$ and ending on $Y_i^+$.
Thus, we identify the max-likelihood transition state $z_{t+k}^{i}$ on $I_{t+k}$ for a cycle walk $\mathcal{G}_i$ starting from $\textbf{q}_{t^+}^i$ in $I_t$: 
\setlength{\abovedisplayskip}{4.0pt}
\setlength{\belowdisplayskip}{4.0pt}
\small\begin{gather}\label{eq:max_likelihood_state}
    p^{\mathcal{G}}_{X_{t+k}|X_t,X_t^+}(j|Y_i^+,Y_i^+) = \sum_{i,l} \frac{A_{t^+}^{t+k}(i,j)A_{t+k}^t(j,l)}{C ||Y_i^+||} \\
    z_{t+k}^{i} = \argmax_{\textbf{q}_{t+k}^j \in \textbf{q}_{t+k}} p^{\mathcal{G}}_{X_{t+k}|X_t,X_t^+}(j| Y_i^+, Y_i^+)  
\end{gather}\normalsize
where $\textbf{q}_{t^+}^i \in  Y_i^+$ and $\textbf{q}_{t}^l \in  Y_i^+$. We identify $\mathcal{Z}_{t+k}^{i} = \mathcal{C}_{t+k}(z_{t+k}^i)$ as the cluster of \acp{roi} on $I_{t+k}$ matching to the cluster $\mathcal{C}_{t^+}(\textbf{q}_{t^+}^i)$ of \acp{roi} on $I_t$ (\cref{fig:method_forward}).

\myparagraph{Optimizing Mutually-exclusive Assignments.}
Given the set of positive nodes $\mathbf{Q}_t^+$
in $I_t$, we propose to enforce the desired mutually-exclusive connectivity property on $\mathcal{G}$ - \ie one cluster $\mathcal{Z}_{t+k}^i$ in $I_{t+k}$ is assigned to at most one $\mathcal{C}_t^i$ on $I_t$ - by incrementally assigning the clusters ${\mathcal{C}_t^i \; \forall \mathbf{q}_{t^+}^i \in \mathbf{Q}_t^+}$ to previously unassigned pseudo-matches $\mathcal{Z}_{t+k}^i$ in $I_{t+k}$, and optimizing the corresponding transition probabilities. 
In particular, (i) we sort the unique clusters $\mathcal{C}_t^i$ by their cycle closure probability {\small${p^{\mathcal{G}}_{X_t|X_t^+}(Y_i^+|\mathcal{C}_t^i) = \frac{1}{||Y^+_i||}\sum_{\mathbf{q}_{t^+}^m \in Y^+_i}\sum_{\mathbf{q}_{t}^i \in Y^+_i}{p^{\mathcal{G}}_{X_t|X_t^+}(l|m)}}$}; 
(ii) since low cycle closure probability means that a latent correspondence cannot be found, we filter out clusters with cycle closure probability less than a threshold $\beta_{\text{cycle}}$, \ie {\small${\mathcal{C}_t^{\text{valid}} = \{ \mathcal{C}_t^i \; | \; p^{\mathcal{G}}_{X_t|X_t^+}(Y_i^+|\mathcal{C}_t^i) \ge \beta_{\text{cycle}} = 0.8; \; \forall \mathbf{q}_{t^+}^i \in \mathbf{Q}_t^+\}}$};
(iii) for each valid cluster ${\mathcal{C}_t^i \in \mathcal{C}_t^{\text{valid}}}$ in $I_t$ we find a matching cluster $\mathcal{Z}_{t+k}^i \not \in \mathcal{Z}_{t+k}^{\text{assigned}}$ in $I_{t+k}$ that was not previously matched to another cluster, where $\mathcal{Z}_{t+k}^{\text{assigned}}$ is the set of already-assigned latent clusters;
(iv) we optimize the forward transition probabilities $A_{t^+}^{t+k}$ using an $L_2$ loss, whose positive targets for nodes in a cluster ${\mathcal{C}_t^l \in \mathcal{C}_t^{\text{valid}}}$ are $\mathcal{Z}_{t+k}^i$, and all other nodes are negative targets:
\setlength{\abovedisplayskip}{4.0pt}
\setlength{\belowdisplayskip}{4.0pt}
\small\begin{align}
    \mathcal{L}_{\text{forward}} &= 
    \sum_{l, i, j} (p^{\mathcal{G}}_{X_{t+k} | X_t^+}(j|i) - I[\mathbf{q}^j_{t+k} \in \mathcal{Z}_{t+k}^i])^2 = \\
    &=   \sum_{l, i, j} (A_{t^+}^{t+k}(i,j) - I[\mathbf{q}^j_{t+k} \in \mathcal{Z}_{t+k}^i])^2,
\end{align}\normalsize
where {\small$\{ l, i, j | \mathcal{C}_{t^+}^l \in \mathcal{C}_{t^+}^{\text{valid}}, \mathbf{q}_{t^+}^i \in {\mathcal{C}_{t^+}^l}, \mathbf{q}^j_{t+k} \in \mathbf{Q}_{t+k}\}$} and $I[\cdot]$ is the indicator function.
We sample three times more negative pairs than positive ones to balance the loss.

\subsection{Total Loss} \label{ssec:method_total_loss} 
We optimize the entire network under
{\small$\mathcal{L}_{\text{total}} = \mathcal{L}_{\text{det}} + \gamma_1\mathcal{L}_{\text{cycle}} + \gamma_2\mathcal{L}_{\text{forward}}$}.
{\small$\mathcal{L}_{\text{det}}$} is the loss for the chosen object detector on the key frame $I_t$, and {\small$\gamma_1=1.0$}, {\small$\gamma_2=2.0$}.

\subsection{Tracking with Walker} \label{ssec:method_walker_matching}
We here detail Walker's inference-time data association pipeline used for tracking with a \ac{toag} trained as in \cref{ssec:method_cycle,ssec:method_forward,ssec:method_total_loss}.

\myparagraph{Biwalk Similarity.} 
Inspired by the properties of our cyclic (bi-directional) walk on temporal object appearance graphs, we propose the \textit{biwalk}, a novel appearance similarity metric.
Let $N$ be the set of detected objects in frame $I_t$ with appearance embeddings
$\mathbf{n}$, and $M$ the matching candidates from the past $K$ frames with appearance embeddings $\mathbf{m}$.
We define $\mathcal{G} \! : N \! \rightarrow \!  M \!  \rightarrow \!  N$ as the cycle transition walk from the detections to the matching candidates and back to the detections. $\mathcal{G}$ is described by the cycle transition matrix $\bar{A}_N^N = A_N^M A_M^N$, with $A_N^M$ and $A_M^N$ forward and backward transition matrices respectively.
We then propose to measure the similarity between a detection $N_i$ and a matching candidate $M_j$ as the probability of traversing the corresponding node over a satisfied cycle transition $\mathcal{G}_i: N_i \rightarrow M \rightarrow N_i$. Analogously to \cref{ssec:method_forward}, the biwalk similarity can be used to determine the most-plausible match in $M$ as the max-likelihood transition state on the cyclic graph $\mathcal{G}_i$.
We thus define the biwalk similarity $s_{i,j}^{\text{biwalk}}$ between a detection $i$ and a matching candidate $j$ as:
\setlength{\abovedisplayskip}{4.0pt}
\setlength{\belowdisplayskip}{4.0pt}
\small\begin{align} \label{eq:biwalk_similarity}
s_{i,j}^{\text{biwalk}} &= p^{\mathcal{G}_i}_{M|N,N}(j|i,i) \cdot \text{I}[p^{\mathcal{G}_i}_{N|N}(i|i) \ge \beta_{\text{cycle}}] = \\
                   &= A_{N}^{M}(i,j)A_{M}^N(j,i) / C \cdot \text{I}[ \bar{A}_N^N(i,i) \ge \beta_{\text{cycle}}],
\end{align}\normalsize
where {\small$p^{\mathcal{G}_i}_{M|N,N}(j|i,i) = A_{N}^{M}(i,j)A_{M}^N(j,i) / C$} as shown in \cref{ssec:method_forward}. The higher $s_{i,j}^{\text{biwalk}}$, the stronger the similarity. 
Enforcing that the cycle transition is satisfied - \ie ${p^{\mathcal{G}_i}_{N|N}(i|i) \ge \beta_{\text{cycle}}}$ - allows to reject false positive matches.
We ablate on the superiority of our biwalk similarity over other appearance match metrics in \cref{ssec:exp_ablation}.

\myparagraph{Data Association.} Inspired by BYTE~\cite{zhang2022bytetrack}, we adopt a two-stage data association scheme.
In our first association stage, we propose to associate high-confidence detections to tracklets based on the max-likelihood transition state under motion-constrained bi-directional walks. 
We then follow the original BYTE implementation for the second association stage.
Pseudo-code in the Appendix.

We here describe in details our first association stage.
We define a novel gating function $W$ for Hungarian assignment of detections to matching candidates based on motion-constrained appearance similarity. 
In particular, we combine our appearance similarity metric \textit{biwalk} with spatial proximity between the detected objects $N$ and the matching candidates $M$ refined by Kalman filtering. 

First, we adopt the Kalman filter~\cite{kalman1960new} to predict the future location of the matching candidates.
We estimate the \textit{motion cost} via the \ac{iou} distance $d_{i,j}^{\text{IoU}}=1-IoU(M_i,N_j)$ between the i-th predicted bounding box and j-th detected one.
We estimate the \textit{appearance cost} via the biwalk distance $d_{i,j}^{\text{biwalk}} = 1 - s_{i,j}^{\text{biwalk}}$.
Similarly to~\cite{aharon2022bot}, we reject appearance-based matches for objects that are spatially far-apart - \ie $d_{i,j}^{\text{IoU}} \ge \beta_{\text{IoU}}$ - or with dissimilar appearance - \ie $d_{i,j}^{\text{biwalk}} \ge \beta_{\text{biwalk}}$ - by setting their cost to 1:
\setlength{\abovedisplayskip}{4.0pt}
\setlength{\belowdisplayskip}{4.0pt}
\small\begin{align} \label{eq:cost_matrix_walker_plus_plus}
    \hat{d}_{i,j}^{\text{biwalk}}= 
        \begin{cases}
            d_{i,j}^{\text{biwalk}}, 
            & \text{if } (d_{i,j}^{\text{IoU}} < \beta_{\text{IoU}}) \wedge (d_{i,j}^{\text{biwalk}} < \beta_{\text{biwalk}}) \\
            1,              & \text{otherwise}
        \end{cases}
\end{align}\normalsize

Finally, we fuse the appearance- $\hat{d}_{i,j}^{\text{biwalk}}$ and motion-based $d_{i,j}^{\text{IoU}}$ costs as their element-wise minimum: {\small${W_{i,j} = \min\{\lambda_{\text{biwalk}} \cdot \hat{d}_{i,j}^{\text{biwalk}}, d_{i,j}^{\text{IoU}}\}}$}, with $\lambda_{\text{biwalk}}$ relative weight of the appearance cost wrt. motion.
We use the fused cost matrix $W$ for Hungarian assignment of detections to matching candidates.

\section{Experiments}
We provide details on our evaluation protocol for self-supervised \ac{mot} methods (\cref{ssec:exp_evaluation_protocol}).
We report implementation details in \cref{ssec:exp_implementation_details}. %
We compare our method with the state of the art in \ac{mot} on sparsely (\cref{ssec:exp_sota_sparse}) and densely (\cref{ssec:exp_sota_dense}) annotated videos. Finally, we conduct ablation studies in \cref{ssec:exp_ablation}.

\subsection{Evaluation Protocol} \label{ssec:exp_evaluation_protocol}
We aim to evaluate the effectiveness of self-supervised \ac{mot} methods for learning appearance and their sensitivity to different annotation sparsity levels.

\myparagraph{Datasets.} 
\textit{MOT17}~\cite{dendorfer2021motchallenge} is one of the most popular pedestrian tracking datasets, annotated at 14 $\sim$ 30 FPS and featuring 7 training and 7 test sequences in crowded street scenes. 
\textit{DanceTrack}~\cite{sun2022dancetrack} is a challenging tracking dataset for pedestrians in uniform appearance and diverse motion. Annotated at 20 FPS, it includes 40 videos for training, 25 for validation, and 35 for testing. Its appearance uniformity provides a challenging setting for appearance-based trackers, and even more for self-supervised ones.
\textit{BDD100K}~\cite{yu2020bdd100k} is a driving dataset annotated at 5 FPS, counting 1400
sequences for training, 200 for validation, and 400 for testing. Featuring 8 classes, it allows to validate \ac{mot} methods in a multi-class setting.
We report the most popular metrics for each dataset.

\myparagraph{Annotation Sparsity.}
We evaluate self-supervised \ac{mot} under two detection annotation settings during training, \ie dense and sparse. 
Tracking labels are never provided.
In the \textit{sparse} setting, detection annotations are provided for only one every $k$ frames. This is the most practical setting, as it is undesirable to annotate all frames in a video. 
We thus compare self-supervised trackers trained with detection annotations at 0.1 FPS, a value sensitively below the minimal annotation rate in tracking datasets (1 FPS~\cite{dave2020tao}) and sparser than the average object living time in a video.
In the \textit{dense} setting, detection annotations are provided for all frames to compare self-supervised to supervised \ac{mot}.

\myparagraph{Self-supervised Baselines.} 
We evaluate all models using the \yolox detector, a 4conv-1fc \textit{embedding head}, and QDTrack's~\cite{fischer2022qdtrack} appearance-only data association scheme.
First, we compare across all settings to QDTrack-S~\cite{fischer2022qdtrack}, which uses data augmentation for image-level self-supervision.
Then, we ablate against the self-supervised Re-ID literature (\cref{tab:ablation_ss_reid}) by extending MvMHAT~\cite{gan2021self} and ReMOTS~\cite{yang2020remots} to the joint detection and tracking setting. Moreover,  
Moreover, we apply the original contrastive random walk for pixel correspondences~\cite{jabri2020space} on our quasi-dense \ac{toag} defined in \cref{ssec:method_appearance_graph}. We refer to it as QD-CRW.
Finally, we introduce an appearance-only variant of Walker that follows QDTrack's data association scheme, namely QD-Walker. 
Details in the Appendix.

\begin{table}[t]
\centering
\setlength{\tabcolsep}{3.5pt}
\scriptsize
\caption{\textbf{State of the art on DanceTrack.} We compare existing methods on DanceTrack's test set under sparse (0.1 FPS) and dense (20 FPS) annotations. Methods in black use self-supervised appearance.}
\label{tab:dancetrack-sota}
\begin{tabular}{@{}p{0.5cm}clccccc@{}}
\toprule
                                                                                                       & Self. Sup.  & Method                              & HOTA & AssA & DetA & MOTA & IDF1 \\ \midrule
\multirow{3}{*}{\raisebox{+0\normalbaselineskip}[0pt][0pt]{\rotatebox[origin=c]{90}{\textbf{Sparse}}}} 
                                                                                                       &     \multirow{3}{*}{\cmark}  &     QDTrack-S~\cite{fischer2022qdtrack} &    29.2 &    12.3 &    70.2 &    79.3 &    22.6 \\
                                                                                                       &                              &    QD-Walker (ours)                       &      41.0 &      23.2 &    \textbf{72.6} &      85.8 &    39.9 \\
                                                                                                       &                              &     Walker (ours)                     &    \textbf{45.9} &    \textbf{29.5} &      71.9 &    \textbf{86.2} &    \textbf{49.0} \\ \midrule
\multirow{10}{*}{\raisebox{+0\normalbaselineskip}[0pt][0pt]{\rotatebox[origin=c]{90}{\textbf{Dense}}}} & \gr \multirow{7}{*}{\xmark}  & \gr FairMOT~\cite{zhang2021fairmot}     & \gr 39.7 & \gr 23.8 & \gr 66.7 & \gr 82.2 & \gr 40.8 \\
                                                                                                       &                              & \gr CenterTrack~\cite{zhou2020tracking} & \gr 41.8 & \gr 22.6 & \gr   78.1 & \gr 86.8 & \gr 35.7 \\
                                                                                                       &                              & \gr TransTrack~\cite{sun2020transtrack} & \gr 45.5 & \gr 27.5 & \gr 75.9 & \gr 88.4 & \gr 45.2 \\
                                                                                                       &                              & \gr ByteTrack~\cite{zhang2022bytetrack} & \gr 47.7 & \gr 32.1 & \gr 71.0 & \gr   89.6 &  \gr  53.9 \\
                                                                                                       &                              & \gr QDTrack~\cite{fischer2022qdtrack}   &  \gr  54.2 & \gr   36.8 & \gr 80.1 & \gr 87.7 & \gr 50.4 \\
                                                                                                       &                              & \gr MOTR~\cite{zeng2022motr}            &  \gr  54.2 & \gr 40.2 & \gr 73.5 & \gr 79.7 & \gr 51.5 \\
                                                                                                       &                              & \gr OC-SORT~\cite{cao2022observation}   &  \gr 55.1 & \gr   38.3 &  \gr 80.3 & \gr 92.0 & \gr   54.6 \\ \cmidrule{2-8}
                                                                                                       &     \multirow{3}{*}{\cmark}  &   QDTrack-S                       &   38.3 &   19.8 &   77.2 &   85.4 &   33.6 \\
                                                                                                       &                              &   QD-Walker (ours)                       &   49.8 &   32.2 &   \textbf{77.3} &   89.4 &   49.3 \\
                                                                                                       &                              &    Walker (ours)                     &   \textbf{52.4} &   \textbf{36.1} &   76.5 &     \textbf{89.7} &     \textbf{55.7} \\ \bottomrule
\end{tabular}
\end{table}

\subsection{Implementation Details} \label{ssec:exp_implementation_details}
In the \textit{sparse} setting, we select positive nodes for our appearance graph (\cref{ssec:method_appearance_graph}) by their IoU with high-confidence detections, and with the available ground-truth boxes in the \textit{dense} setting.
We train Walker using a batch size of 16 and an initial learning rate of 0.00025, decayed with a cosine schedule after a one-epoch warm-up. We initialize the detector from a COCO pre-trained model. We train on 8 GPUs NVIDIA RTX 3090.
On MOT17, we follow the private detector half-train/half-val protocol, training for 50 epochs on the union of CrowdHuman~\cite{shao2018crowdhuman} and MOT17~\cite{fischer2022qdtrack,zhang2022bytetrack,cao2022observation}.
On DanceTrack and BDD100K, we train for 12 and 25 epochs.
On MOT17, we apply offline tracklet interpolation~\cite{zhang2022bytetrack,fischer2022qdtrack,aharon2022bot}.

\subsection{Sparse Annotations - Comparison with the State of the Art} \label{ssec:exp_sota_sparse}
The sparse setting is the most relevant for assessing self-supervised \ac{mot} (\cref{ssec:exp_evaluation_protocol}).
We here consider a 0.1 FPS annotation rate and ablate on the effect of different annotation sparsity rates on self-supervised trackers in \cref{ssec:exp_ablation}.

\myparagraph{Dancetrack.} 
DanceTrack challenges appearance-based trackers by featuring dancing people with uniform appearance. 
While previous work~\cite{zeng2022motr,fischer2022qdtrack} shows that supervised methods can rely on fine details to learn meaningful appearance, the same has never been shown for self-supervised ones.
Our experiments (\cref{tab:dancetrack-sota}, \textbf{Sparse}) show that Walker and QD-Walker significantly outperform QDTrack-S by +16.7 HOTA~\cite{luiten2021hota} and with more than twice the \ac{assa} (29.5 vs. 12.3).
We argue that Walker's remarkable improvement over QDTrack-S is due to its access to the unlabeled video stream during self-supervision, which allows Walker to learn how to match under the rapid pose changes across DanceTrack's neighboring frames. Since QDTrack-S is only exposed to individual frames during training, it cannot deal with rapid pose changes.

\myparagraph{BDD100K.} Similar observations hold for 
BDD100K (\cref{tab:bdd-sota}, \textbf{Sparse}). Walker learns more discriminative multi-class appearance descriptors than QDTrack-S.

\begin{table}[t]
\centering
\setlength{\tabcolsep}{3.5pt}
\scriptsize
\caption{\textbf{State of the art on BDD100K.} We compare with existing methods on the BDD100K test set under sparse (0.1 FPS) and dense (5 FPS) annotations. Methods in black use self-supervised appearance.}
\label{tab:bdd-sota}
\begin{tabular}{@{}p{0.5cm}clcccc@{}}
\toprule
                                                                                                       & Self. Sup.  & Method                              & mMOTA & mIDF1 & MOTA & IDF1 \\ \midrule
\multirow{3}{*}{\raisebox{+0\normalbaselineskip}[0pt][0pt]{\rotatebox[origin=c]{90}{\textbf{Sparse}}}} 
                                                                                                       &     \multirow{3}{*}{\cmark}  &   QDTrack-S ~\cite{fischer2022qdtrack} &   37.1 &   49.7  &   63.5 &    64.0  \\
                                                                                                       &                              &   QD-Walker  (ours)                       &     37.8 &     52.3  &     64.7 &     67.2   \\
                                                                                                       &                              &   Walker   (ours)                     &     \textbf{39.0} &     \textbf{54.1}  &     \textbf{68.2} &     \textbf{70.1}    \\ \midrule
\multirow{8}{*}{\raisebox{+0\normalbaselineskip}[0pt][0pt]{\rotatebox[origin=c]{90}{\textbf{Dense}}}}  & \gr \multirow{5}{*}{\xmark}  & \gr Yu \etal~\cite{yu2020bdd100k}       & \gr 26.3  & \gr 44.7  & \gr 58.3 & \gr 68.2 \\
                                                                                                       &                              & \gr DeepSORT~\cite{wojke2017simple}     & \gr 31.6  & \gr 38.7  & \gr 56.9 & \gr 56.0 \\
                                                                                                       &                              & \gr TETer~\cite{li2022tracking}         & \gr 37.4  & \gr 53.3  & \gr -    & \gr -    \\
                                                                                                       &                              & \gr ByteTrack~\cite{zhang2022bytetrack} & \gr  40.1  &  \gr  55.8  &  \gr  69.9 & \gr   71.3 \\
                                                                                                       &                              & \gr QDTrack~\cite{fischer2022qdtrack}   & \gr  42.4  &  \gr  55.6  &  \gr  68.4 & \gr   73.9 \\ \cmidrule{2-7}
                                                                                                       &     \multirow{3}{*}{\cmark}  &   QDTrack-S ~\cite{fischer2022qdtrack}     &   38.7 &   50.3      &   65.2 &     66.8     \\
                                                                                                       &                              &   QD-Walker  (ours)           &   39.6 &   53.4      &   65.9 &     69.7     \\
                                                                                                       &                              &   Walker   (ours)         &    \textbf{41.2} &     \textbf{56.1}      &     \textbf{68.3} &     \textbf{72.1}    \\ \bottomrule
\end{tabular}
\end{table}

\subsection{Dense Annotations - Comparison with the State of the Art} \label{ssec:exp_sota_dense}
Although Walker learns appearance representations in a self-supervised way, we show that it impressively reports competitive performance with the supervised state of the art on MOT17~\cite{dendorfer2021motchallenge}, DanceTrack~\cite{sun2022dancetrack}, and BDD100K~\cite{yu2020bdd100k}. Walker's training follows the dense protocol (\cref{ssec:exp_evaluation_protocol}). 

\myparagraph{Dancetrack.} 
(\cref{tab:dancetrack-sota}, \textbf{Dense}) shows that our self-supervised appearance-only Walker outperforms several popular trackers, including ByteTrack. Its high-quality appearance representations make Walker competitive with other supervised methods such as QDTrack~\cite{fischer2022qdtrack} and MOTR~\cite{zeng2022motr}, even achieving the highest IDF1 across all methods.

\myparagraph{BDD100K.}
On the multi-class dataset BDD100K (\cref{tab:bdd-sota}, \textbf{Dense}), Walker outperforms the supervised appearance-based TETer~\cite{li2022tracking} and improves over ByteTrack~\cite{zhang2022bytetrack}, demonstrating the importance of appearance descriptors in tracking.

\myparagraph{MOT17.}
The relatively linear motion of pedestrians in MOT17 (\cref{tab:mot-sota}) makes the benchmark particularly suitable for motion-based trackers. Nevertheless, our self-supervised appearance-only baseline QD-Walker approaches supervised appearance-only trackers such as QDTrack and MOTR, and the full Walker further improves it and reports competitive performance.

\begin{table}[t]
\centering
\setlength{\tabcolsep}{3.5pt}
\scriptsize
\caption{\textbf{State of the art on MOT17.} We compare  methods with private detectors on MOT17's test set under dense annotations (14 $\sim$ 30 FPS). Methods in black use self-supervised appearance.}
\label{tab:mot-sota}
\begin{tabular}{@{}p{0.5cm}clccccc@{}}
\toprule
 & Self. Sup.    & Method                           & HOTA & AssA & DetA & MOTA & IDF1 \\ \midrule
\multirow{12}{*}{\raisebox{+0\normalbaselineskip}[0pt][0pt]{\rotatebox[origin=c]{90}{\textbf{Dense}}}} & \gr \multirow{9}{*}{\xmark}  & \gr CenterTrack~\cite{zhou2020tracking} & \gr 52.2 & \gr 51.0 & \gr 53.8 & \gr 67.8 & \gr 64.7 \\
   &      & \gr FairMOT ~\cite{zhang2021fairmot}  & \gr 59.3 & \gr 58.0 & \gr 60.9 & \gr 73.7 & \gr 72.3 \\
 &      & \gr TransTrack~\cite{sun2020transtrack}   & \gr 54.1 & \gr 47.9 & \gr 61.6 & \gr 63.9 & \gr 74.5 \\
 &      & \gr ByteTrack~\cite{zhang2022bytetrack}   & \gr 63.1 & \gr 62.0 &  \gr  64.5 &\gr  77.3 & \gr   80.3 \\
 &      & \gr QDTrack~\cite{fischer2022qdtrack}    &\gr  63.5 & \gr 62.6 &  \gr  64.5 & \gr   78.7 & \gr 77.5 \\
 &      & \gr MOTR~\cite{zeng2022motr}    & \gr 57.8 & \gr 55.7 & \gr 60.3 & \gr 68.6 & \gr 73.4 \\
 &      & \gr OC-SORT~\cite{cao2022observation}   & \gr 63.2 &  \gr  63.2 & \gr -    & \gr 77.5 &\gr  78.0 \\
 &      &\gr  StrongSORT++~\cite{du2023strongsort}  &  \gr  64.4 & \gr  64.4 &  \gr 64.6 & \gr  79.5 & \gr   79.6 \\
 &      & \gr BoT-SORT~\cite{aharon2022bot}    &  \gr  64.6 & \gr -    & \gr -    & \gr 79.5 & \gr 80.6 \\\cmidrule{2-8}
   &\multirow{3}{*}{\cmark}  & QDTrack-S~\cite{fischer2022qdtrack}  &  58.9 &  59.2 &  62.6 &  74.4 &  74.0 \\
   &      & QD-Walker  (ours)     &  61.7 &  60.6 &  63.1 &  75.4 &  74.2 \\
   &      & Walker   (ours)        &    \textbf{63.6} &    \textbf{63.0} &  \textbf{64.0} &    \textbf{78.2} &  \textbf{77.4} \\ \bottomrule
\end{tabular}
\end{table}

\noindent\begin{minipage}{0.48\textwidth}

\centering
\setlength{\tabcolsep}{2.5pt}
\scriptsize
\captionof{table}{Comparison to self-supervised Re-ID ($\dagger$) and self-supervised correspondence ($\ddagger$) approaches on DanceTrack val. For a fair comparison, all baselines share the same architecture and inference algorithm as our appearance-only QD-Walker.
}
\label{tab:ablation_ss_reid}
\begin{tabular}{@{}llccc@{}}
\\
\\
\toprule
                                                                                                       & Method                & HOTA          & AssA          & DetA  \\ \midrule
\multirow{5}{*}{\raisebox{+0.\normalbaselineskip}[0pt][0pt]{\rotatebox[origin=c]{90}{\textbf{Sparse}}}} & QD-CRW$^{\ddagger}$~\cite{jabri2020space}    & 18.4          & 4.8           &  \textbf{72.7}   \\
                                                                                                       & MvMHAT$^{\dagger}$~\cite{gan2021self}    &   40.7          &   23.4          & 71.6          \\
                                                                                                       & ReMOTS$^{\dagger}$~\cite{yang2020remots}    &   41.0          &   23.5          &   71.8             \\
                                                                                                       & QD-Walker (Ours)         &  42.2 &  24.7 &   71.7       \\
                                                                                                       & Walker (Ours)         &  \textbf{47.6}    &  \textbf{31.0} &  71.5\\ \midrule
\multirow{5}{*}{\raisebox{+0.\normalbaselineskip}[0pt][0pt]{\rotatebox[origin=c]{90}{\textbf{Dense}}}}  & QD-CRW$^{\ddagger}$~\cite{jabri2020space}    & 19.2          & 5.1           & 74.1         \\
                                                                                                       & MvMHAT$^{\dagger}$~\cite{gan2021self}    &   44.6          &   26.9          &  \textbf{75.0}   \\
                                                                                                       & ReMOTS$^{\dagger}$~\cite{yang2020remots}    &   45.2          &   27.5          &   74.8         \\
                                                                                                       & QD-Walker (Ours)         &  49.0 &  32.8 &   73.6     \\
                                                                                                       & Walker (Ours)         &  \textbf{53.0} &  \textbf{38.6} &  73.1       \\ \bottomrule
\\
\\
\\
\vspace{-0.3em}  %
\end{tabular}

\end{minipage}
\quad
\begin{minipage}{0.48\textwidth}


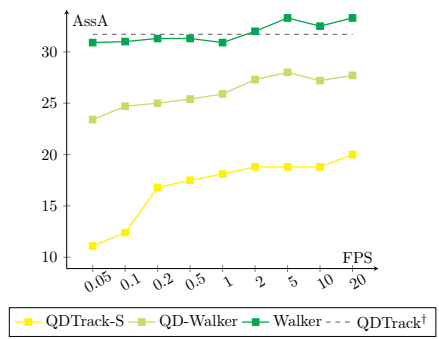
\captionof{figure}{Self-supervised \ac{mot} under different annotation sparsity rates (FPS) during training. We compare video-level (Walker; QD-Walker) and frame-level (QDTrack-S) self-supervision. $\dagger$: reference QDTrack fully-supervised at 20 FPS.} \label{fig:annotation_sparsity_assa}

\centering
    \resizebox{\textwidth}{!}{
\begin{tikzpicture}[scale=0.8]
\begin{axis}[
axis lines=middle,
ymin=11,
ymax=32,
legend style ={at={(0.5,-0.15)}, anchor=north, legend columns=-1},
xlabel=FPS,
ylabel=AssA,
xticklabel style = {rotate=30,anchor=north},
enlargelimits = true,
xticklabels from table={dancetrack_assa.dat}{FPS},xtick=data]
\addplot[colorTrd,thick,mark=square*] table [y=QDTrack-S,x=X]{dancetrack_assa.dat};
\addlegendentry{QDTrack-S}
\addplot[colorSnd,thick,mark=square*] table [y=QD-Walker,x=X]{dancetrack_assa.dat};
\addlegendentry{QD-Walker}
\addplot[colorFst,thick,mark=square*] table [y=Walker,x=X]{dancetrack_assa.dat};
\addlegendentry{Walker}]
\addplot[gray,dashed,thick] table [y=QDTrack,x=X]{dancetrack_assa.dat};
\addlegendentry{QDTrack$^\dagger$}]
\end{axis}
\end{tikzpicture}
}
 
\end{minipage}

\subsection{Ablation Study} \label{ssec:exp_ablation}
\myparagraph{Annotations Sparsity.}
We argue that a good self-supervised \ac{mot} method must fully utilize the available unlabeled data to learn meaningful appearance representations. Thus, we compare in \cref{fig:annotation_sparsity_assa} the sensitivity to different annotation sparsity levels during training for representative self-supervised \ac{mot} methods.
We compare: QDTrack-S~\cite{fischer2022qdtrack}, which relies on image-level self-supervision by augmenting static images; QD-Walker, which shares the same architecture and appearance-only tracking algorithm with QDTrack-S but utilizes video-level self-supervision; Walker, which further combines motion cues to appearance ones.
All methods use YOLOX as detector.
We assess their \ac{assa} at different annotation frame rates - varying from 0.05 to 20 FPS - on the DanceTrack validation set.
We find that our video-level self-supervision is considerably more robust to annotation sparsity, and it can outperform image-level self-supervision even when reducing the number of annotated frames by 400x. 
Moreover, complementing appearance with motion, Walker's performance remains remarkably stable at any annotation frame rate, outperforming the fully supervised QDTrack despite not using tracking labels and even with up to 10x less annotated frames.

\myparagraph{Self-supervised Re-ID.}
As motivated in \cref{ssec:exp_evaluation_protocol}, we extend baselines from the self-supervised Re-ID~\cite{gan2021self,yang2020remots} and correspondence learning~\cite{jabri2020space} literature to the joint detection and tracking problem. 
For a fair comparison, all methods share the same architecture and inference algorithm as the appearance-only QD-Walker. Walker additionally uses motion to reject unlikely appearance-based associations. 
Compared to all other baselines, both QD-Walker and Walker show stark superiority in association accuracy, proving the superiority of our self-supervised appearance-learning scheme. Moreover, the comparison to QD-CRW indicates that the original single-positive contrastive random walk is suboptimal on quasi-dense \acp{toag}. We argue that: (i) a single positive formulation introduces several false negatives in the optimized loss; (ii) by not enforcing mutual exclusivity its assignments are ambiguous for MOT, where one detection must be assigned to at most one tracklet. This further validates the importance of our contributions towards learning an optimal \acp{toag} topology for \ac{mot}.

\begin{figure*}[t]
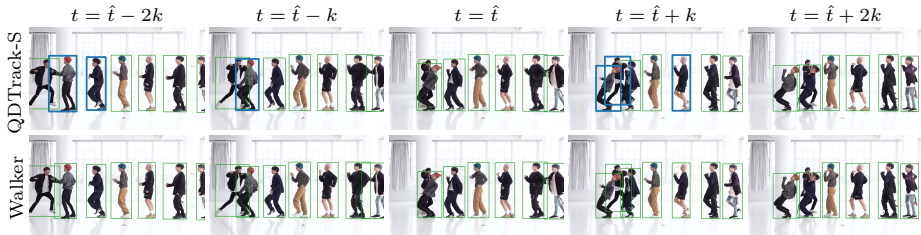

\centering
\scriptsize
\setlength{\tabcolsep}{1pt}
\begin{tabular}{cccccc}
 & $t=\hat{t}-2k$ & $t=\hat{t}-k$  & $t=\hat{t}$  & $t=\hat{t}+k$  & $t=\hat{t}+2k$ \\
\raisebox{+2\normalbaselineskip}[0pt][0pt]{\rotatebox[origin=c]{90}{QDTrack-S}} & \includegraphics[width=0.19\textwidth,trim={6.0cm 4.0cm 10.0cm 4.0cm},clip]{figures/supplement/vis/dancetrack/0058/errors/static_qdtrack/000120.jpg} & \includegraphics[width=0.19\textwidth,trim={6.0cm 4.0cm 10.0cm 4.0cm},clip]{figures/supplement/vis/dancetrack/0058/errors/static_qdtrack/000124.jpg} & \includegraphics[width=0.19\textwidth,trim={6.0cm 4.0cm 10.0cm 4.0cm},clip]{figures/supplement/vis/dancetrack/0058/errors/static_qdtrack/000128.jpg} & \includegraphics[width=0.19\textwidth,trim={6.0cm 4.0cm 10.0cm 4.0cm},clip]{figures/supplement/vis/dancetrack/0058/errors/static_qdtrack/000132.jpg} & \includegraphics[width=0.19\textwidth,trim={6.0cm 4.0cm 10.0cm 4.0cm},clip]{figures/supplement/vis/dancetrack/0058/errors/static_qdtrack/000136.jpg} \\
\raisebox{+2\normalbaselineskip}[0pt][0pt]{\rotatebox[origin=c]{90}{Walker}}    & \includegraphics[width=0.19\textwidth,trim={6.0cm 4.0cm 10.0cm 4.0cm},clip]{figures/supplement/vis/dancetrack/0058/errors/walker_plus_plus/000120.jpg}  & \includegraphics[width=0.19\textwidth,trim={6.0cm 4.0cm 10.0cm 4.0cm},clip]{figures/supplement/vis/dancetrack/0058/errors/walker_plus_plus/000124.jpg} & \includegraphics[width=0.19\textwidth,trim={6.0cm 4.0cm 10.0cm 4.0cm},clip]{figures/supplement/vis/dancetrack/0058/errors/walker_plus_plus/000128.jpg} & \includegraphics[width=0.19\textwidth,trim={6.0cm 4.0cm 10.0cm 4.0cm},clip]{figures/supplement/vis/dancetrack/0058/errors/walker_plus_plus/000132.jpg} & \includegraphics[width=0.19\textwidth,trim={6.0cm 4.0cm 10.0cm 4.0cm},clip]{figures/supplement/vis/dancetrack/0058/errors/walker_plus_plus/000136.jpg}
\end{tabular}
\vspace{-1em}
  \caption{We analyze 5 frames spaced by 0.2 seconds of the DanceTrack sequence \textit{0058}. Compared to image-level self-sup. (QDTrack-S~\cite{fischer2022qdtrack}), Walker effectively utilizes the temporal information to reduce ID switches (blue). Correctly tracked boxes in green.}  
  \label{fig:vis_main_dancetrack_errors_0058}
\end{figure*}

\myparagraph{Qualitative Results}. We analyze 5 frames spaced by 0.2 seconds of the DanceTrack sequence 0058 (\cref{fig:vis_main_dancetrack_errors_0058}). Walker eliminates the ID switches caused by occlusions and rapid pose changes, further validating that - unlike QDTrack-S - Walker can effectively learn to disambiguate non-rigid objects under rapidly varying poses by learning from the temporal stream.

\section{Conclusion}
This paper introduces Walker, the first self-supervised multiple object tracker that learns from sparse detection annotations and no instance IDs. 
Walker self-supervises appearance representations by optimizing the topology of a cleverly-designed temporal object appearance graph (\cref{ssec:method_appearance_graph}). We let meaningful instance similarities (edges) emerge by optimizing our multi-positive contrastive random walks (\cref{ssec:method_cycle}), and enforce the mutually-exclusive graph connectivity necessary to downstream association (\cref{ssec:method_forward}).
By relying on video-level self-supervision, Walker effectively makes use of the unlabeled frames in sparsely annotated datasets.
As a result, Walker significantly outperforms previous state-of-the-art self-supervised trackers~\cite{fischer2022qdtrack} even when trained with 400x less annotated frames.
Remarkably, Walker is the first self-supervised tracker to achieve competitive performance with state-of-the-art supervised trackers on a variety of benchmarks.
We hope that our work will inspire future research in downstream tracking applications dealing with limited labels, \eg open-world and open-vocabulary tracking~\cite{yang2023track,li2023ovtrack}, domain adaptation~\cite{segu2023darth}, continual learning~\cite{liu2023cooler}.
Finally, by replacing the commonly-used frame-level self-supervision with our video-level self-supervision, we believe that our contributions will enable training stronger foundational models for multiple object tracking~\cite{li2024matching}.

\clearpage
\section*{Acknowledgements}
This work was supported in part by the Max Plank ETH Center for Learning Systems.

\bibliographystyle{splncs04}
\bibliography{main}

\clearpage
\section*{Appendix}
\def\thesection{\Alph{section}}
\setcounter{section}{0}
We here provide additional details on the method, mathematical proofs, implementation details and experimental results.
In \cref{app:sec:proof_transition} we provide a mathematical derivation of the latent transition probability on a random walk.
Refer to \cref{app:ssec:walker_training} for details on Walker's training.
Refer to \cref{app:ssec:walker_inference} and \cref{app:ssec:qd_walker_inference} for an in-depth explanation of the inference tracking algorithms of Walker and QD-Walker respectively.
\cref{app:sec:implementation_details} reports implementation details.
In \cref{app:sec:datasets}, we motivate the use of a sparse setting and stress the usefulness of self-supervised trackers leveraging the temporal information to make full use of such label-efficient settings.
Finally, we report additional quantitative and qualitative (\cref{app:ssec:exp_add_results}) results.

\section{Latent Transition Probability Derivation} \label{app:sec:proof_transition}
We prove \cref{eq:transition_probability} (\cref{ssec:method_forward}), which represents the probability of transitioning on a given latent node $\textbf{q}_{t+k}^j$ given that a walk on the appearance graph $\mathcal{G}$ (\cref{ssec:method_appearance_graph}) starts from $\textbf{q}_{t^+}^i$ and ends in $\textbf{q}_{t}^l$.

Let $\mathcal{G}: \textbf{Q}_t^+ \rightarrow \textbf{Q}_{t+k} \rightarrow \textbf{Q}_t$ be a cyclic temporal graph connecting the nodes $\textbf{Q}_t^+$ in the frame $I_{t}$ to $\textbf{Q}_{t+k}$ in the frame $I_{t+k}$ and back to $\textbf{Q}_t$ in $I_{t}$.
$\mathcal{G}$ is a Markov chain described by the forward and backward transitions $A_{t^+}^{t+k}$ and ${A}_{t+k}^{t}$, whose chained transition $\bar{A}_{t^+}^{t}$ describes the cycle correspondence as a multi-step walk along the appearance graph $\mathcal{G}$. 
Let $X_t$ be the state of a walker at time $t$, and $p_{X_t}(i)$ the probability of being at node $i$ at time $t$.
\begin{theorem} \label{app:lemma:transition_probability}
The probability of transitioning on a latent node $\textbf{q}_{t+k}^j$ on the reference image $I_{t+k}$ when starting from $\textbf{q}_{t^+}^i$ in $I_t$ and ending on $\textbf{q}_{t}^l$ in $I_t$ along the cycle walk $\mathcal{G}$ is:
\setlength{\abovedisplayskip}{4.0pt}
\setlength{\belowdisplayskip}{4.0pt}
\fontsize{8.5}{5}\begin{align}
p^{\mathcal{G}}_{X_{t+k}|X_t,X_t^+}(j|l,i)
    &= p^{\mathcal{G}}_{X_t|X_{t+k}}(l|j)p^{\mathcal{G}}{X_{t+k}|X_t^+}(j|i) / C  \\
    &= A_{t^+}^{t+k}(i,j)A_{t+k}^t(j,l) / C
\end{align}\normalsize
where {\small$C = \sum_{\textbf{q}_{t+k}^m \in \textbf{Q}_{t+k}} p^{\mathcal{G}}_{X_t| X_{t+k}}(l|m) p^{\mathcal{G}}_{X_{t+k} | X_t^+}(m|i)$} is a normalizing constant.
\end{theorem}

\begin{proof}[Proof of \cref{app:lemma:transition_probability}]
\setlength{\abovedisplayskip}{4.0pt}
\setlength{\belowdisplayskip}{4.0pt}
\fontsize{8.5}{5}\begin{align*} 
p^{\mathcal{G}}_{X_{t+k}|X_t,X_t^+}(j|l,i) &\stackrel{(1)}{=} \frac{p^{\mathcal{G}}_{X_t,X_{t+k},X_t^+}(l,j,i)}{p^{\mathcal{G}}_{X_t,X_t^+}(l,i)} \\
    &\stackrel{(2)}{=} \frac{p^{\mathcal{G}}_{X_t|X_{t+k}}(l|j)p^{\mathcal{G}}_{X_{t+k}|X_t^+}(j|i)}{p^{\mathcal{G}}{X_t|X_t^+}(l|i)p^{\mathcal{G}}{X_t^+}(i)} \\
    &\stackrel{(3)}{=} \frac{p^{\mathcal{G}}_{X_t|X_{t+k}}(l|j)p^{\mathcal{G}}_{X_{t+k}|X_t^+}(j|i)}{\sum_{\textbf{q}_{t+k}^m \in \textbf{Q}_{t+k}}p^{\mathcal{G}}_{X_t|X_{t+k}}(l|m)p^{\mathcal{G}}_{X_{t+k}|X_t}(m|l)p^{\mathcal{G}}_{X_t^+}(i)} \\
    &\stackrel{(4)}{=} p^{\mathcal{G}}_{X_t|X_{t+k}}(l|j)p^{\mathcal{G}}_{X_{t+k}|X_t^+}(j|i) \; / \;  C \\
    &\stackrel{(5)}{=} A_{t^+}^{t+k}(i,j)A_{t+k}^t(j,l) \;  / \;  C
\end{align*}\normalsize

\noindent We here motivate the steps in the proof:

(1) by the definition of conditional probability.

(2) since a walk on the appearance graph $\mathcal{G}$ defined in \cref{ssec:method_appearance_graph} is a first-order Markov chain, each transition only depends on the previous state, \ie{}  {\small${p^{\mathcal{G}}_{X_t,X_{t+k},X_t^+}(l,j,i) = p^{\mathcal{G}}_{X_t|X_{t+k}}(l|j)p^{\mathcal{G}}_{X_{t+k}|X_t^+}(j|i)}$}.

(3) since a walk on the appearance graph $\mathcal{G}$ defined in \cref{ssec:method_appearance_graph} is a first-order Markov chain, each transition only depends on the previous state, \ie{} {\small${p^{\mathcal{G}}_{X_t|X_t^+}(l|i) = \sum_{\textbf{q}_{t+k}^m \in \textbf{Q}_{t+k}}p^{\mathcal{G}}_{X_t|X_{t+k}}(l|m)p^{\mathcal{G}}_{X_{t+k}|X_t}(m|l)}$}. Moreover, we marginalize over all possible transition states {\small$\textbf{q}_{t+k}^m \in \textbf{Q}_{t+k}$}.

(4) for a chosen starting node $\textbf{q}_{t}^i$ and ending node $\textbf{q}_{t}^l$, \\ {\small${C= \sum_{\textbf{q}_{t+k}^m \in \textbf{Q}_{t+k}}p^{\mathcal{G}}_{X_t|X_{t+k}}(l|m)p^{\mathcal{G}}_{X_{t+k}|X_t}(m|l)p^{\mathcal{G}}_{X_t^+}(i)}$} is a normalization constant.

(5) according to our definition of the transition probability matrices for a random walk on an appearance graph $\mathcal{G}$ (\cref{ssec:method_cycle}).

\end{proof}

\vspace{-1em}
\section{Training a Walker} \label{app:ssec:walker_training}

\begin{algorithm}[ht]
\caption{Training pipeline of Walker for identifying and optimizing pseudo-assignments.} \label{alg:training}
\small
\begin{algorithmic}[1]
    \INPUT detections $\mathcal{D}_t$ at time $t$ and detections $\mathcal{D}_{t+k}$ at time $t+k$, ground-truth detections $\hat{\mathcal{D}}_t$ at time $t$ and ground-truth detections $\hat{\mathcal{D}}_{t+k}$ at time $t+k$, setting \texttt{setting} (\textit{dense} or \textit{sparse}) 
    \State \AlgoComment{embed detections}
    \State $\mathbf{Q}_{t}$ = \texttt{embed}($\mathcal{D}_t$)
    \State $\mathbf{Q}_{t+k}$ = \texttt{embed}($\mathcal{D}_{t+k}$)
    \State \AlgoComment{select reference nodes for walk based on setting}
    \If{\texttt{setting} == \textit{dense}} 
        \State $\bar{\mathcal{D}}_{t}$ = $\hat{\mathcal{D}}_{t}$ 
        \State $\bar{\mathcal{D}}_{t+k}$ = $\hat{\mathcal{D}}_{t+k}$ 
    \ElsIf{\texttt{setting} == \textit{sparse}}
        \State \AlgoComment{filter detections by confidence}
        \State $\bar{\mathcal{D}}_{t}$ = \texttt{filterByConf}($\mathcal{D}_t$, $\beta_{obj}$)
        \State $\bar{\mathcal{D}}_{t+k}$ = \texttt{filterByConf}($\mathcal{D}_{t+k}$, $\beta_{obj}$)
    \EndIf
    \State \AlgoComment{negative-positive balance for walk nodes}
    \State $\mathbf{Q}_{t}$ = ($\mathbf{Q}_{t}^+$, $\mathbf{Q}_{t}^-$) = \texttt{negPosBalance}(($\mathbf{Q}_{t}$, $\mathcal{D}_t$), \texttt{gt}=$\bar{\mathcal{D}}_{t}$, \texttt{neg\_pos\_rate}=3)
    \State $\mathbf{Q}_{t+k}$ = ($\mathbf{Q}_{t+k}^+$, $\mathbf{Q}_{t+k}^-$) = \texttt{negPosBalance}(($\mathbf{Q}_{t+k}$, $\mathcal{D}_{t+k}$), \texttt{gt}=$\bar{\mathcal{D}}_{t+k}$, \texttt{neg\_pos\_rate}=3)
    \State
    \AlgoComment{compute cycle probabilities}
    \State $A_{t^+}^{t+k}$ = \texttt{computeTransition}($\mathbf{Q}_t^+,\mathbf{Q}_{t+k}$)
    \State $A_{t+k}^{t}$ = \texttt{computeTransition}($\mathbf{Q}_{t+k},\mathbf{Q}_t$) 
    \State $\bar{A}_{t^+}^{t}$ = \texttt{concatTransitions}($A_{t}^{t+k},A_{t+k}^{t}$)
    \State
    \AlgoComment{get valid clusters}
    \State $\mathcal{C}_t$ = \texttt{getClusters}($\mathbf{Q}_{t}^+$)
    \State $\mathcal{C}_t$ = \texttt{set}($\mathcal{C}_t^{\text{high}}$)  \LineComment{keep only unique clusters}
    \State $\mathcal{C}_t$ = \texttt{sorted}($\mathcal{C}_t$, \texttt{key}=$\bar{A}_{t^+}^{t}$)
    \State $\mathcal{C}_t^{\text{valid}}$ = \texttt{filterByConf}($\mathcal{C}_t$, $\bar{A}_{t^+}^{t}$, $\beta_{cycle}$)
    \State \AlgoComment{find pseudo-assignments}
    \State $\mathcal{Z}_{t+k}^{\text{assigned}} = [ \; ]$ \LineComment{set of assigned clusters}
    \For{$\mathcal{C}_t^{i}$ \texttt{in} $\mathcal{C}_t^{\text{valid}}$}
        \State \AlgoComment{find match not in $\mathcal{Z}_{t+k}^{\text{assigned}}$}
        \State $\mathcal{Z}_{t+k}^i$ = \texttt{findMatch}($A_{t^+}^{t+k}$, $\mathcal{C}_t^{i}$, $\mathcal{Z}_{t+k}^{\text{assigned}}$)
        \State $\mathcal{Z}_{t+k}^{\text{assigned}}$.\texttt{append}($\mathcal{Z}_{t+k}^i$)
    \EndFor
    \State \AlgoComment{compute losses}
    \State $\mathcal{L}_{\text{cycle}}$ = \texttt{cycleLoss}($\bar{A}_{t^+}^{t}$, $\mathcal{C}_t^{\text{high}}$)
    \State $\mathcal{L}_{\text{forward}}$ \kern-0.2em = \kern-0.2em \texttt{forwardLoss}($A_{t^+}^{t+k}$, \kern-0.2em $\mathcal{C}_t^{\text{valid}}$, \kern-0.2em $\mathcal{Z}_{t+k}^{\text{assigned}}$)
    \State $\mathcal{L}_{\text{total}}$ = $\mathcal{L}_{\text{cycle}}$ + $\mathcal{L}_{\text{forward}}$
\end{algorithmic}
\end{algorithm}

We here provide additional details on Walker's training, which we introduced in \cref{ssec:method_cycle}, \cref{ssec:method_forward}, \cref{ssec:method_total_loss}. In particular, we discussed our multi-positive contrastive random walk in \cref{ssec:method_cycle}, our cluster-wise forward assignment and optimization in \cref{ssec:method_forward}, and the total loss in \cref{ssec:method_total_loss}.
To make the understanding of our training pipeline easier, we provide pseudo-code in \cref{alg:training}.

\myparagraph{Node Embedding.} During one training iteration, we are given the detections $\mathcal{D}_t$ on $I_t$ and $\mathcal{D}_{t+k}$ on $I_{t+k}$, and the ground-truth detections $\hat{\mathcal{D}}_t$ on $I_t$ and $\hat{\mathcal{D}}_{t+k}$ on $I_{t+k}$.
We first embed the detections to obtain their embeddings, \ie $\mathbf{q}_t$ and $\mathbf{q}_{t+k}$ respectively.

\myparagraph{Node Selection.} Depending on the setting - \ie \textit{dense} or \textit{sparse} (see \cref{ssec:exp_evaluation_protocol}) - we use different policies for selecting the positive and negative nodes in each frame.
Note that we defined in \cref{ssec:method_appearance_graph} the positive nodes as the ones with high \ac{iou} with a set of reference bounding boxes $\bar{\mathcal{D}}_t$.
In the \textit{sparse} setting, we cannot assume the detection annotations to be available for both key and reference frame. Thus, we use the high-confidence detections ${\mathcal{D}}_t^{\text{high}}$ as set of reference bounding boxes $\bar{\mathcal{D}}_t = {\mathcal{D}}_t^{\text{high}}$.
In the \textit{dense} setting, detection annotations are available for all frames. We can thus reliably identify good nodes over which performing our contrastive random walk as the nodes overlapping with the ground truth bounding boxes $\hat{\mathcal{D}}_t$, \ie $\bar{\mathcal{D}}_t = \hat{\mathcal{D}}_t$.
Given the reference bounding boxes $\bar{\mathcal{D}}_t$, we sample positive and negative nodes with a rate of $1/3$.

\myparagraph{Cluster Assignment.} We then compute the forward $A_{t^+}^{t+k}$, backward $A_{t+k}^{t}$, and cycle  $\bar{A}_{t^+}^{t}$ transition probabilities (\cref{ssec:method_cycle}).
We obtain the set of unique clusters $\mathcal{C}_t$ in the key frame $I_t$, and sort and filter them by their cluster cycle probability, ensuring that it must be higher than a threshold $\beta_{\text{cycle}}$.
Finally, we incrementally match key clusters to reference clusters $\mathcal{Z}_{t+k}^i$ based on their max-likelihood transition state, as introduced in \cref{eq:max_likelihood_state}.

\myparagraph{Total Loss.} The pseudo-assignments identified with the algorithm described above are then optimized with the forward loss $\mathcal{L}_{\text{forward}}$, applied jointly with the cycle loss $\mathcal{L}_{\text{cycle}}$.

\section{Tracking with Walker} \label{app:ssec:walker_inference}

We introduced Walker's tracking scheme in \cref{ssec:method_walker_matching}.
To make the understanding of our matching pipeline with fused motion and appearance easier, we provide in \cref{alg:inference_plus_plus} the matching pseudo-code for a whole video $\texttt{V}$.

Inspired by BYTE~\cite{zhang2022bytetrack}, Walker adopts a two-stage matching scheme.
Let $\mathcal{T}$ be the tracklets of the video up to time $t-1$.
Let $\texttt{Det}$ be the object detector.
Let $I_t$ be the incoming frame at time $t$.
$\mathcal{D}_t = \texttt{Det}(I_t)$ is the set of detections predicted by the object detector on $I_t$.
We define the set of high-confidence detections ${\mathcal{D}_t^{\text{high}} = \{d_t^i \in \mathcal{D}_t \; | \; \text{conf}(d_t^i) \ge \beta_{\text{high}}\}}$ as those with confidence greater than a threshold $\beta_{\text{high}}$,
and the set of low-confidence detections ${\mathcal{D}_t^{\text{low}} = \{d_t^i \in \mathcal{D}_t \; | \; \beta_{\text{low}} \le \text{conf}(d_t^i) < \beta_{\text{high}}\}}$ as those with confidence between thresholds $\beta_{\text{low}}$ and $\beta_{\text{high}}$.

In the \textit{first association stage}, Walker matches high-confidence detections $\mathcal{D}_t^{\text{high}}$ to tracklets $\mathcal{T}$ based on our cost matrix $W$ (defined in \cref{eq:cost_matrix_walker_plus_plus}) that fuses motion and appearance costs. 
In the \textit{second association stage}, low confidence detections $\mathcal{D}_t^{\text{low}}$ are assigned to the remaining tracklets $\mathcal{T}_{\text{remain}}$ based on their \ac{iou}.
Unmatched tracklets $\mathcal{T}_{\text{unmatched}}$ are deleted, and new tracklets are initialized from the remaining high-confidence detections $\mathcal{D}_t^{\text{remain}}$.

Track rebirth~\cite{wojke2017simple,zhou2020tracking} is not shown in the algorithm for simplicity.
For additional details on the track management scheme, refer to BYTE~\cite{zhang2022bytetrack}. 
\begin{algorithm}[ht]
\caption{Inference pipeline of Walker for associating objects across a video sequence.}
    \label{alg:inference_plus_plus}
\small
    \begin{algorithmic}[1]
        \INPUT{A video sequence $\texttt{V}$; object detector $\texttt{Det}$}
        \OUTPUT{Tracks $\mathcal{T}$ of the video}
        
        \State Initialization: $\mathcal{T} \leftarrow \emptyset$
        \For{frame $I_t \in \texttt{V}$}
            \State \AlgoComment{predict detection boxes \& scores}
            \State $\mathcal{D}_t \leftarrow \texttt{Det}(I_t)$ 
            \State $\mathcal{D}_t^{\text{high}} \leftarrow \emptyset$ 
            \State $\mathcal{D}_t^{\text{low}} \leftarrow \emptyset$ 
            \For{$d_t^i \in \mathcal{D}_t$}
               \If{$\text{conf}(d_t^i) \ge \beta_{\text{high}}$}
                   \State $\mathcal{D}_t^{\text{high}} \leftarrow  \mathcal{D}_t^{\text{high}} \cup \{d_t^i\}$ 
               \ElsIf{$\beta_{\text{low}} \le \text{conf}(d_t^i) < \beta_{\text{high}}$}
                   \State $\mathcal{D}_t^{\text{low}} \leftarrow  \mathcal{D}_t^{\text{low}} \cup \{d\}$ 
               \EndIf
            \EndFor
            
                \State \AlgoComment{predict new locations of tracks}
            \For{$t \in \mathcal{T}$}
                \State $t \leftarrow \texttt{KalmanFilter}(t)$ 
            \EndFor
        	
            \State \AlgoComment{first association}
            \State Associate $\mathcal{T}$ and $\mathcal{D}_t^{\text{high}}$ using $W^{++}$ (\cref{eq:cost_matrix_walker_plus_plus}) and match threshold $\beta_{\text{match}}^{\text{high}}$  
            \State $\mathcal{D}_t^{\text{remain}} \leftarrow \text{remaining object boxes from } \mathcal{D}_t^{\text{high}}$ 
            \State $\mathcal{T}_{\text{remain}} \leftarrow \text{remaining tracks from } \mathcal{T}$ 
            
                \State \AlgoComment{second association}
            
            \State Associate $\mathcal{T}_{\text{remain}}$ and $\mathcal{D}_t^{\text{low}}$ using IoU distance  and match threshold $\beta_{\text{match}}^{\text{low}}$
            
            \State $\mathcal{T}_{\text{unmatched}} \leftarrow \text{remaining tracks from } \mathcal{T}_{\text{remain}}$

            \State \AlgoComment{delete unmatched tracks}
            $\mathcal{T} \leftarrow \mathcal{T} \setminus \mathcal{T}_{\text{unmatched}}$ 
        	
            \State \AlgoComment{initialize new tracks}
            \For{$d_t^j \in \mathcal{D}_t^{\text{remain}}$}
        	\State $\mathcal{T} \leftarrow  \mathcal{T} \cup \{d_t^j\}$ 
            \EndFor
        \EndFor
        \RETURN{$\mathcal{T}$} 
    \end{algorithmic}
\end{algorithm}

\section{Tracking with QD-Walker} \label{app:ssec:qd_walker_inference}
We briefly introduced QD-Walker's appearance-only tracking scheme in \cref{ssec:exp_evaluation_protocol}.
The goal was to provide an appearance-only tracking baseline that could be used to directly compare to QDTrack-S and other self-supervised Re-ID baselines to establish which self-supervised appearance learning schemes translates to the better tracker.

To make the understanding of our appearance-only matching pipeline easier, we provide pseudo-code for one matching step (\cref{alg:inference}).
Inspired by QDTrack~\cite{pang2021quasi,fischer2022qdtrack}, Walker matches detections to tracklets based on their appearance. However, as opposed to QDTrack's bisoftmax, Walker uses the biwalk similarity metric $s_{i,j}^{biwalk}$ introduced in \cref{eq:biwalk_similarity} between the embeddings of the i-th detection and the j-th tracklet.

We borrow from QDTrack the track management scheme to keep track of inactive and currently active tracks and to handle the matching of objects. Active tracks are tracks that have a matching detection in the previous frame, otherwise they become inactive. Tracks that are inactive for $K$ frames will be removed and not be considered for matching. 
In particular, Walker first removes duplicate detections with inter-class \ac{nms} with confidence threshold {\small$\texttt{Det. Conf. Thr.}$} and IoU threshold {\small$\texttt{Det. NMS IoU. Thr.}$}.
Detections are only considered for matching to existing tracks if the detection confidence is above
a threshold $\beta_{\text{obj}}$. 
A match is determined if the biwalk similarity $s_{i,j}^{biwalk}$ is
higher than a threshold $\beta_{\text{match}}$. For unmatched objects that have a detection confidence higher than a threshold $\beta_{\text{new}}$, we initialize a new track instead. 
We keep the unmatched objects as backdrops for $L$ frames and use them as matching candidates. Detections that are matched to backdrops will thus not be matched to existing tracks. 
The tracklet embeddings are updated with an exponential moving average with momentum $m$.
For additional details on the track management scheme, refer to the original QDTrack paper~\cite{pang2021quasi}. 

\begin{algorithm}[ht]
    \caption{Inference pipeline of QD-Walker for associating objects across a video sequence.}
    \label{alg:inference}
\small
    \begin{algorithmic}[1]
        \INPUT frame index $t$, detections $\textbf{b}_i$, scores $s_i$, detection embeddings $\textbf{n}_i$ for $i=1 \ldots N$, and track embeddings $\textbf{m}_j$ for $j=1 \ldots M$
        \State \AlgoComment{compute matching scores}
        \State \texttt{DuplicateRemoval}($\textbf{b}_i$)
        \For{$i=1\ldots N, j=1\ldots M$} 
        \State \textbf{f}$(i, j) = $ \texttt{biwalk}($\textbf{n}_i, \textbf{m}_j$)
        \EndFor
        \State \AlgoComment{track management}
        \For{$i=1\ldots N$} 
        \State $c$ = \texttt{max}$\left(\textbf{f}(i)\right)$ \LineComment{match confidence}
        \State $j_{\texttt{match}}$ = \texttt{argmax}$\left(\textbf{f}(i)\right)$ \LineComment{matched track ID}
        \State \AlgoComment{object match found}
        \If{$c$ > $\beta_{\texttt{match}}$ \textbf{and} $s_i >$ $\beta_{\texttt{obj}}$}
        \Statex \hskip1.5em \textbf{and} \texttt{isNotBackdrop}$\left(j_{\texttt{match}}\right)$ 
        \State \AlgoComment{update track}
        \State \texttt{updateTrack}$\left(j_{\texttt{match}}, \textbf{b}_i, \textbf{n}_i, t\right)$ 
        \ElsIf{$s_i >$ $\beta_{\texttt{new}}$} 
        \State \AlgoComment{create new track}
        \State \texttt{createTrack}$\left(\textbf{b}_i, \textbf{n}_i, t\right)$
        \Else{}
        \State \AlgoComment{add new backdrop}
        \State \texttt{addBackdrop}$\left(\textbf{b}_i, \textbf{n}_i, t \right)$
        \EndIf
        \EndFor
    \end{algorithmic}
\end{algorithm}

\section{Implementation Details}
\label{app:sec:implementation_details}
We report the training and inference hyperparameters for Walker in \cref{tab:hyperparameters}, identified by parameter-search on the validation set of each dataset.
Since our inference algorithm builds on top of BYTE and QDTrack, we take their hyperparameters directly unless differently specified.
Notice that Walker shares the same trained model and parameters with QD-Walker, only the inference scheme differs.
Results reported for other trackers are directly taken from their papers or re-run following the hyperparameters introduced in the respective paper.

\myparagraph{Training Hyperparameters}
The key frame is sampled from the set of frames with bounding box annotations, \ie in the sparse setting we assume that one frame every $k$ is labeled starting from the first frame in the video sequence.
We sample the reference frame from a neighborhood of the key frame, where the neighborhood width is $\hat{k}$.
For data augmentation, we utilize mosaic augmentation on key and reference frame, followed by consistent photometric augmentations as in~\cite{zhang2022bytetrack}. We then apply non-consistent multi-scale resizing augmentations on key and reference frame, with a scale range (0.5, 1.5) around the basic image size 1440 x 800.

\myparagraph{Inference Hyperparameters}
We report the inference hyperparameters for Walker following the naming convention established throughout the paper, and re-iterated in \cref{app:ssec:walker_training} and \cref{app:ssec:walker_inference}.

\begin{table}[]
\centering
\scriptsize
\setlength{\tabcolsep}{5pt}
\caption{\textbf{Hyper-parameters used in each benchmark.} We include both training and inference parameters of Walker across all datasets.}
\label{tab:hyperparameters}
\begin{tabular}{@{}llccc@{}}
\toprule
                                                                                                        & Parameter                                                    & MOT17 & DanceTrack & BDD100K \\ \midrule
\raisebox{-2.5\normalbaselineskip}[0pt][0pt]{\rotatebox[origin=c]{90}{\textbf{Training}}}                 & $\lambda_1$                                                  & 1.0   & 1.0        & 0.5     \\
                                                                                                        & $\lambda_2$                                                  & 2.0   & 2.0        & 1.0     \\
                                                                                                        & $\hat{k}$                                                    & 10    & 10         & 3       \\
                                                                                                        & $\alpha_1$                                         & 0.7   & 0.7        & 0.7     \\
                                                                                                        & $\alpha_2$                                         & 0.3   & 0.3        & 0.3     \\
                                                                                                        & $\beta_{\text{obj}}$                                         & 0.3   & 0.3        & 0.3     \\
                                                                                                        & $\beta_{\text{cycle}}$                                       & 0.8   & 0.8        & 0.8     \\
                                                                                                        & $\tau$                                                       & 0.05  & 0.05       & 0.05    \\ \midrule
\raisebox{-6\normalbaselineskip}[0pt][0pt]{\rotatebox[origin=c]{90}{\textbf{Inference}}}                 & $\texttt{Det. Conf. Thr.}$                         & 0.1   & 0.1        & 0.1     \\
                                                                                                        & $\texttt{Det. NMS IoU Thr.}$                            & 0.7   & 0.6        & 0.65    \\
                                                                                                        & $\beta_{\text{new}}$                                         & 0.75  & 0.8        & 0.5     \\
                                                                                                        & $\beta_{\text{high}}$                                        & 0.3   & 0.6        & 0.35    \\
                                                                                                        & $\beta_{\text{low}}$                                         & 0.1   & 0.1        & 0.1     \\
                                                                                                        & $\beta_{\text{match}}^{\text{high}}$ & 0.1   & 0.1        & 0.1     \\
                                                                                                        & $\beta_{\text{biwalk}}$ & 0.2   & 0.2        & 0.2     \\
                                                                                                        & $\beta_{\text{IoU}}$                                         & 0.5   & 0.5        & 0.5     \\
                                                                                                        & $\lambda_{\text{biwalk}}$                                    & 2.0   & 2.0        & 2.0     \\
                                                                                                        & $\beta_{\text{match}}^{\text{low}}$                          & 0.5   & 0.5        & 0.5     \\
                                                                                                        & $\beta_{\text{cycle}}$                                       & 0.1   & 0.1        & 0.1     \\
                                                                                                        & $\tau$                                                       & 0.07  & 0.07       & 0.07    \\
                                                                                                        & $K$                                                          & 30    & 20         & 10      \\
                                                                                                        & $m$                                                          & 0.5   & 0.8        & 0.8     \\ \bottomrule
\end{tabular}
\end{table}

\section{Datasets and Annotations} \label{app:sec:datasets}
The minimal annotation frame rate found across tracking datasets is 1 FPS~\cite{dave2020tao}. Under this cut-off value, annotating tracking is often not possible due to the limited living span of objects in a video.  
For this reason, the TAO dataset~\cite{dave2020tao} was originally annotated at 1 Hz.
NuScenes~\cite{caesar2020nuscenes} is annotated only at 2 Hz due to the difficulty in calibration and syncronization of multiple sensors.
However, the large differences in appearance across the sparsely annotated frames in such datasets makes it difficult to learn supervised trackers.
For this reason, the TAO dataset~\cite{dave2020tao} was later refined to 6 FPS~\cite{athar2023burst}.
By not requiring instance labels, a good self-supervised tracker would achieve good tracking performance even when trained under a sparse annotation regimen, as it could make use of the unlabeled frames.

For this reason, we choose to evaluate self-supervised trackers trained with detection annotations at 0.1 FPS (\cref{ssec:exp_evaluation_protocol}), a value sensitively below the common annotation rate and often sparser than the average object living time in a video.
Note that on MOT17 we only validate the dense protocol due to the very small size of its half-train set (only 7 videos totalling 2658 frames).
Self-supervised tracking methods leveraging temporal self-supervision can make full use of the video stream, even in correspondence of the unlabeled frames, overcoming the limitations of training supervised trackers on sparsely annotated data. 
Moreover, by learning from such a low annotation frame rate, self-supervised multiple object tracking algorithms such as Walker allow to significantly reduce the annotation cost for video datasets.
Finally, Walker can be in principle extended to fully unlabeled videos. Given a pre-trained object detector, Walker can be used to train the embedding head on the unlabeled videos while keeping the detector frozen or finetuning it with knowledge distillation techniques.
We leave this interesting application to future work.

\section{Additional Results}\label{app:ssec:exp_add_results}

\subsection{Additional Self-supervised Re-ID baselines}
We compare to additional self-supervised Re-ID baselines~\cite{bastani2021self,ho2020two,wang2021different}. Since such methods do not provide an official implementation, or they cannot be easily extended to an appearance-only setting, we compare Walker against their published results.

In \cref{tab:ablation_public_detections}, we compare on MOT17's public Faster R-CNN detections against Bastani \etal~\cite{bastani2021self} and Ho \etal~\cite{ho2020two}.
Walker greatly outperforms both approaches, showing the superiority of our self-supervised appearance representations.

In \cref{tab:comparison_crw}, we compare against a straightforward extension of the CRW to Re-ID by directly using the CRW module trained for point correspondence as an object-level Re-ID module.
Although their performance is satisfying (albeit greatly supported by the two-stage pipeline and motion-based heuristics of the JDE's algorithm), the appearance representations learned by the original CRW algorithm are not object-specific and do not enforce mutual exclusivity.
By addressing both limitations, Walker achieves higher performance.

In \cref{tab:hard}, we report the performance compared on \textit{all} and \textit{top-5 hardest} sequences from DanceTrack val. We compare Walker to ByteTrack~\cite{zhang2022bytetrack} and ByteTrack~+~\cite{jabri2020space}.
We choose ByteTrack as a representative motion-only tracker, and naively extend it with a pre-trained \cite{jabri2020space} as Re-ID head. Since \cite{jabri2020space} was trained on the DAVIS dataset, ByteTrack + \cite{jabri2020space} drastically fails to cope with DanceTrack's similar object appearances, worsening ByteTrack's association (\cref{tab:hard}).

\begin{table}[t]
\centering
\scriptsize
\setlength{\tabcolsep}{5pt}
\caption{\textbf{Comparison to baselines on public detections.} We compare to existing baselines which report results on the public detection set of MOT17. For a fair comparison, we use Faster R-CNN and train only on MOT17, without using Crowdhuman.}
\label{tab:ablation_public_detections}
\begin{tabular}{@{}llll@{}}
\toprule
Method         & MOTA & IDF1 & MOTP      \\ \midrule
Bastani et al.~\cite{bastani2021self} & 56.8 & 58.3 & \textbf{-} \\
Ho et al.~\cite{ho2020two}      & 48.1 & -    & 76.7       \\
Walker         & \textbf{68.0} & \textbf{64.5} & \textbf{78.4}    \\ \bottomrule
\end{tabular}
\end{table}

\begin{table}[t]
\centering
\scriptsize
\setlength{\tabcolsep}{5pt}
\caption{\textbf{Comparison to CRW as Re-ID.} We compare Walker on the MOT17 validation set against the CRW used as a Re-ID module in a JDE~\cite{wang2020towards} tracker as in \cite{wang2021different}. Both methods combine appearance with motion.}
\label{tab:comparison_crw}
\begin{tabular}{@{}lll@{}}
\toprule
Method                                       & HOTA      & IDF1      \\ \midrule
JDE-CRW~\cite{wang2021different} & 61.7      & 73.0      \\
Walker (Ours)                  & \textbf{63.6} & \textbf{77.4} \\ \bottomrule
\end{tabular}
\end{table}

\subsection{Ablation on Method Details}
We here ablate on the method components that leverage the quasi-dense nature of our temporal object appearance graph. In particular, we ablate on (i) the effectiveness of the proposed method components, (ii) the effect of different appearance-based match metrics, (iii) the use of a single-positive vs. a multi-positive contrastive cycle consistency objective, and (iv) the importance of enforcing mutually-exclusive assignments.

\myparagraph{Method Components.} 
We ablate on the effectiveness of each proposed component (\cref{tab:method_components}) on top of the naive quasi-dense contrastive random walk baseline (QD-CRW).
We incrementally add our multi-positive contrastive objective (+ multi-positive), enforce mutually-exclusive connectivity (+ mutually-exclusive), replace the bisoftmax similarity with our biwalk match metric in QDTrack's appearance-only inference (+ biwalk), and add motion constraints to reject unlikely appearance-based associations (+ motion).
While all rows in \cref{tab:method_components} learn from our proposed \ac{toag}, our contributions clearly promote an optimal graph topology for \ac{mot} (5.1 vs. 38.6 AssA).

\noindent\begin{minipage}{0.55\textwidth}

\centering
\setlength{\tabcolsep}{2.5pt}
\scriptsize
\captionof{table}{Ablation on our individual method components on top of the naive quasi-dense CRW (QD-CRW) on DanceTrack val.}
\label{tab:method_components}
\begin{tabular}{@{}lccc@{}}
\toprule
Method                & HOTA          & AssA          & DetA         \\ \midrule
QD-CRW                & 19.2          & 5.1           &   \textbf{74.1}     \\
+ multi-positive      & 46.3          & 30.2          & 71.8              \\
+ mutually-exclusive  &   47.3          &   31.3          & 71.7             \\
+ biwalk (= QD-Walker)   &   49.0          &   32.8          &   73.6              \\
+ motion (= Walker) &   \textbf{53.0} &   \textbf{38.6} &   73.1         \\ \bottomrule
\end{tabular}

\end{minipage}
\quad
\begin{minipage}{0.4\textwidth}

\centering
\setlength{\tabcolsep}{2.5pt}
\scriptsize
\captionof{table}{Ablation on match metrics for appearance-only tracking (QD-Walker) on DanceTrack val.}
\label{tab:ablation_match_metric}
\begin{tabular}{@{}lccc@{}}
\\

\toprule
Metric                              & HOTA & AssA & DetA \\ \midrule
Cosine                              &  47.3 &  31.3 &  71.7 \\
Bisoftmax~\cite{fischer2022qdtrack} &  46.8 &  30.6 &  71.7 \\
Biwalk                              &  \textbf{49.0} &  \textbf{32.8} &   \textbf{73.6} \\ \bottomrule
\\
\end{tabular}

\end{minipage}

\myparagraph{Appearance-based Match Metrics.}
We ablate on the effect of different appearance-based similarity metrics in appearance-only MOT with QD-Walker (\cref{tab:ablation_match_metric}). Our proposed biwalk improves the overall tracking performance.

\myparagraph{Multi-positive Cyclic Contrastive Objective.}
In \cref{tab:cycle_targets}, we ablate on different formulations of our cycle consistency formulation introduced in \cref{ssec:method_cycle}.
We report the results for cycle walks optimized wrt. a single target (a), and multiple targets (b). We find that our proposed multi-positive formulation is remarkably more effective than the naive single positive baseline. We argue that the single-positive baseline treats as negatives for the contrastive loss also all the other nodes expect for the self node that represent detections which are highly overlapping with the target node, and likely to represent the same instance. Consequently, a significant amount of noise is injected in the training, making it more difficult for the embedding head to discriminate instances. We solve this problem with our multi-positive formulation, which enables multiple positive target for each contrastive random walk. 

\begin{table}[]
\centering
\setlength{\tabcolsep}{2.5pt}
\scriptsize
\caption{\textbf{Ablation on the selection policy for the cycle walk targets.} We ablate on the DanceTrack validation set on different options of the target nodes to optimize for a cycle walk $\mathcal{G}_i$ starting from a node $\mathbf{q}_{t^+}^i$ in $I_t$ and ending on $\mathbf{q}_{t^+}^i$ itself. The forward loss is not applied here. Optimizing cycle transitions only with respect to the destination node  $\mathbf{q}_{t^+}^i$ itself (a) considers as negatives also the highly overlapping nodes which are likely to represent the same instance, creating a conflicting self-supervisory signal. This problem is solved by considering as positives all the nodes $Y_i^+$ highly overlapping with  $\mathbf{q}_{t^+}^i$.}
\label{tab:cycle_targets}
\begin{tabular}{@{}llccccc@{}}
\toprule
Selection Policy   & Cycle Prob.             & HOTA          & AssA          & DetA          & MOTA          & IDF1          \\ \midrule
a) Single-positive & $p^{\mathcal{G}}_{X_t | X_t^+}(i|i)$     & 39.6          & 22.8          & 69.4          & 79.1          & 37.4          \\
b) Multi-positive  & $p^{\mathcal{G}}_{X_t | X_t^+}( Y_i^+|i)$ & \textbf{48.7} & \textbf{31.1} & \textbf{77.1} & \textbf{88.9} & \textbf{48.0} \\ \bottomrule
\end{tabular}
\end{table}

\begin{table}[t]
\centering
\scriptsize
\setlength{\tabcolsep}{0.5pt}
\caption{\textbf{Ablation on the selection policy for the match pseudo-labels.} 
We ablate on the DanceTrack validation set  on different formulations of the max-likelihood transition state for a cycle walk $\mathcal{G}_i$ starting from a node $\mathbf{q}_{t^+}^i$ in $I_t$ and ending on $\mathbf{q}_{t^+}^i$ itself in $I_t$ after transitioning on $I_{t+k}$. 
\textit{Single-positive} consists in identifying the max-likelihood transition state on the cycle walk starting from a node $\mathbf{q}_{t^+}^i$ and ending on the node $\mathbf{q}_{t^+}^i$ itself;
\textit{Multi-positive} averages over the multi-positive target nodes $Y_i^+$ for a cycle transition starting in  $\mathbf{q}_{t^+}^i$; 
\textit{Cluster-wise Multi-positive} further averages over the nodes in the starting cluster $\mathcal{C}_t^i=Y_i^+$, and enforces cluster-wise mutually-exclusive assignments with the algorithm described in \cref{ssec:method_forward}.}
\label{tab:cluster_voting}
\begin{tabular}{@{}llcccccc@{}}
\toprule
Selection Policy               & Latent Transition Prob.                                     & HOTA          & AssA          & DetA          & MOTA          & IDF1          \\ \midrule
Single-positive             & {\fontsize{6.5}{5}$p^{\mathcal{G}}_{X_{t+k}|X_t ,X_t^+}(j|i,i)$}                          & 46.2          & 29.1          & 76.8          & 87.0          & 45.9          \\
Multi-positive              & {\fontsize{6.5}{5}$p^{\mathcal{G}}_{X_{t+k}|X_t,X_t^+}(j|Y_i^+,i)$}                      & 46.5          & 30.0          & 76.8          & 86.9          & 46.2          \\
\makecell[l]{Cluster-wise \\ Multi-positive} & {\fontsize{6.5}{5}$p^{\mathcal{G}}_{X_{t+k}|X_t,X_t^+}(j| Y_i^+,Y_i^+)$}                     & \textbf{49.8} & \textbf{32.2} & \textbf{77.3} & \textbf{89.4} & \textbf{49.3} \\ \bottomrule
\end{tabular}
\end{table}

\begin{table}[t]
\centering
\scriptsize
\caption{\textbf{Ablation on the loss components.} We ablate on the DanceTrack validation set  on the importance of each proposed loss components.}
\label{tab:loss_components}
\begin{tabular}{llccccc}
\toprule
$\mathcal{L}_{\text{cycle}}$ & $\mathcal{L}_{\text{forward}}$ & HOTA          & AssA          & DetA          & MOTA          & IDF1          \\ \midrule
\checkmark                     & -                       & 48.7          & 31.1          & 77.1          & 88.9          & 48.0          \\
\checkmark                     & \checkmark                       & \textbf{49.8} & \textbf{32.2} & \textbf{77.3} & \textbf{89.4} & \textbf{49.3} \\ \bottomrule
\end{tabular}
\end{table}

\myparagraph{Mutually-exclusive Forward Assignments.}
In \cref{tab:cluster_voting}, we ablate on different policies to identify and optimize the forward assignments according to the formulation introduced in \cref{ssec:method_forward}.
We report the results with cluster-wise mutually-exclusive assignments (c) and assignments that are not mutually-exclusive (a, b). In particular, (a) uses a single-node to single-node cycle walk formulation to independently identify the max-likelihood latent transition state in the reference frame that matches each node in the key frame. (b) further refines it by averaging the latent transition probabilities over the set of possible targets for the cycle walks departing from a node in the key frame. However, both (a) and (b) consider that each starting node can get independent assignments that are not mutually-exclusive, meaning that nodes in $I_t$ from different instances may be assigned to nodes in $I_{t+k}$ from a same instance, causing conflicts in the optimization. This problem is elegantly addressed by our mutually-exclusive cluster-wise assignment and optimization strategy introduced in  \cref{ssec:method_forward}, which (i) prevents nodes from a same cluster in the key frame to be assigned to nodes in different clusters in the reference frame, and (ii) prevents nodes from different clusters in the key frame to be assigned to a same cluster in the reference frame.

\subsection{Ablation on the Impact of the Hyperparameters}
\begin{table}[]
\centering
\scriptsize
\setlength{\tabcolsep}{4pt}
\caption{Ablation on Walker's sensitive inference parameters on DanceTrack val.}
\label{tab:parameters_ablation}
\begin{tabular}{@{}cccccccc@{}}
\toprule
$\beta_{high}$ & $\beta^{high}_{match}$ & $\lambda_{biwalk}$ & HOTA          & DetA          & AssA          & MOTA          & IDF1          \\ \midrule
0.5            & 0.1                    & 1.0                & 52.6          & 73.1          & 38.1          & 86.9          & 56.4          \\
0.5            & 0.1                    & 2.0                & 53.2          & 73.1          & 38.9          & 87.0          & \textbf{57.2} \\
0.5            & 0.2                    & 2.0                & 52.6          & 73.5          & 37.9          & 86.6          & 55.0          \\
0.6            & 0.1                    & 2.0                & \textbf{53.4} & \textbf{73.6} & \textbf{39.0} & \textbf{87.2} & 56.3          \\ \bottomrule
\end{tabular}
\end{table}

Current tracking-by-detection (TbD) methods, including Walker, are hyperparameter-heavy. 
However, Walker's 14 inference hyperparameters are comparable to state-of-the-art tracking-by-detection methods combining motion and appearance,  \eg BoT-SORT and StrongSORT have 13 according to their official code. 
As mentioned in \cref{app:sec:implementation_details}, our inference algorithm builds on QDTrack and BYTE.
When not explicitly mentioned, we keep all hyperparameters as in their original works.
For the remaining hyperparameters, we conducted a grid search. We here report an analysis of the impact of Walker's most-sensitive inference parameters in \cref{tab:parameters_ablation}.

\subsection{Ablation on Loss Components}
In \cref{tab:loss_components}, we ablate on the importance of the cycle and forward losses towards our total loss introduced in \cref{ssec:method_total_loss}.
We find that applying the forward loss on top of the cycle loss results in a considerable improvement in performance, highlighting the importance of identifying and optimizing max-likelihood latent transition states in a mutually-exclusive fashion according to our proposal in \cref{ssec:method_forward}.
In particular, the performance improvements originates from (i) the quality of the forward assignments refined by averaging over all the walks starting from all the nodes in a given cluster and ending on the multi-positive targets for the corresponding starting node, and (ii) the cluster-wise mutual-exclusivity property enforced as described in \cref{ssec:method_forward}.

\begin{table}[t]
\centering
\scriptsize
\setlength{\tabcolsep}{2pt}
\caption{{\textbf{Comparison to ByteTrack +~\cite{jabri2020space} on DanceTrack dense val.} Performance compared on \textit{all} and \textit{top-5 hardest} sequences. FLOPs are computed using an input size of 3x640x640 to the YOLOX-X detector, of 3x256x128 to [20]'s ResNet-18 Re-ID branch and of 320x7x7 (RoI size in YOLOX-X) to our 4conv-1fc emb. head.}}
\label{tab:hard}
\begin{tabular}{@{}lcccccccc@{}}
\toprule
            & \multicolumn{3}{c}{\textit{All}} &  \multicolumn{3}{c}{\textit{Hard (top-5)}}         &   &       \\ \cmidrule(lr){2-4} \cmidrule(lr){5-7}
Method      & HOTA               & AssA              & DetA              & HOTA        & AssA          & DetA          & \makecell{Det. \\FLOPS (G)}  & \makecell{Re-ID \\FLOPS (G)}      \\ \midrule
ByteTrack~\cite{zhang2022bytetrack}   & 48.9               & 33.1              & 72.4              &  35.6           & 18.8              & 68.0              & 281.9           & -                    \\
ByteTrack +~\cite{jabri2020space} & 24.6                   & 15.1                  & 72.1                  & 17.4            & 7.3              & 68.3              & 281.9           & 1.19 ${\times 10^{-3}}$                 \\
Walker      & \textbf{53.0}          & \textbf{38.6}          & \textbf{73.1}     & \textbf{41.6}            & \textbf{ 25.4}             &  \textbf{69.5}             & \textbf{281.9}           & \textbf{0.14} $\mathbf{\times 10^{-3}}$                 \\

\bottomrule
\end{tabular}
\end{table}

\subsection{Ablation on Model Complexity}
In \cref{tab:hard}, we ablate on the FLOPS requirements of different methods on DanceTrack val. We compare Walker to ByteTrack~\cite{zhang2022bytetrack} and ByteTrack~+~\cite{jabri2020space}.
FLOPs are computed using an input size of 3x640x640 to the YOLOX-X detector, of 3x256x128 to [20]'s ResNet-18 Re-ID branch and of 320x7x7 (RoI size in YOLOX-X) to our 4conv-1fc emb. head.
ByteTrack~+~\cite{jabri2020space} requires a separate R-18 Re-ID head which is $\sim \! \! 9 \times$ more computationally expensive (Re-ID FLOPS, \cref{tab:hard}) than our tiny embedding head, which operates on small-size RoIs and is computationally negligible wrt. the detector.

\subsection{Qualitative Results} \label{app:ssec:exp_qualitative}
We report a qualitative comparison on DanceTrack of the existing self-supervised tracking methods, \ie QDTrack-S~\cite{fischer2022qdtrack}, QD-Walker (ours), and Walker (ours).

\cref{fig:vis_dancetrack_demo_0058,fig:vis_dancetrack_demo_0077,fig:vis_dancetrack_demo_0081} show the tracking results for each method, where the same color is used through time to represent the same ID.
\cref{fig:vis_dancetrack_errors_0058,fig:vis_dancetrack_errors_0077,fig:vis_dancetrack_demo_0081} show the ID switches (blue) and correctly tracked bounding boxes (green).
The qualitative results remark the superiority of Walker over QDTrack-S. By sharing the inference algorithm with QDTrack-S, QD-Walker demonstrates the superiority of our self-supervised appearance-learning algorithm, showing significantly less ID switches under complex occlusions.
This is made possible by our temporal self-supervision in videos, which makes our learned appearance descriptors more robust to the sudden appearance and pose changes in highly dynamic videos as the ones in DanceTrack.
Moreover, the improved tracking algorithm of the full Walker further boosts our tracking performance performance. By taking into account the motion information, Walker notably reduces the number of ID switches in uniform appearance settings such as DanceTrack by constraining matches to only happen near likely future positions of an object.
Notably, \cref{fig:vis_dancetrack_errors_0081} shows a case of rapid object motion and sudden pose changes. For ease of visualization, we crop all frames around an area of interest, \ie where the dancers are thrown in the air. The dynamic evolutions that the dancers are performing make tracking extremely difficult for a self-supervised tracker trained on static images such as QDTrack-S. This can be noticed by the high amount of ID switches (blue boxes). Instead, our trackers trained on the temporal video stream learn appearance representations robust to the temporal pose changes of the dancers, as it can be seen by the significantly better results and reduced ID switches.
It is worth noticing that our motion-constrained tracker (Walker) prevents the ID switch at time $t=\hat{t}$ that still occurs in the unconstrained QD-Walker.
Finally, in \cref{fig:vis_dancetrack_errors_0077} we identify a case where both QD-Walker and Walker cannot remedy an ID switch.
Due to the sudden change in appearance and pose of the dancer, our trackers initiate a new tracklet for an already existing object in $t=\hat{t}-k$.

\clearpage
\begin{figure*}[]
\centering
\scriptsize
\setlength{\tabcolsep}{1pt}
\begin{tabular}{cccccc}
 & $t=\hat{t}-2k$ & $t=\hat{t}-k$  & $t=\hat{t}$  & $t=\hat{t}+k$  & $t=\hat{t}+2k$ \\
\raisebox{+1.9\normalbaselineskip}[0pt][0pt]{\rotatebox[origin=c]{90}{QDTrack-S}} & \includegraphics[width=0.19\textwidth]{figures/supplement/vis/dancetrack/0058/demo/static_qdtrack/000120.jpg} & \includegraphics[width=0.19\textwidth]{figures/supplement/vis/dancetrack/0058/demo/static_qdtrack/000124.jpg} & \includegraphics[width=0.19\textwidth]{figures/supplement/vis/dancetrack/0058/demo/static_qdtrack/000128.jpg} & \includegraphics[width=0.19\textwidth]{figures/supplement/vis/dancetrack/0058/demo/static_qdtrack/000132.jpg} & \includegraphics[width=0.19\textwidth]{figures/supplement/vis/dancetrack/0058/demo/static_qdtrack/000136.jpg} \\
\raisebox{+1.7\normalbaselineskip}[0pt][0pt]{\rotatebox[origin=c]{90}{QD-Walker}}    & \includegraphics[width=0.19\textwidth]{figures/supplement/vis/dancetrack/0058/demo/walker/000120.jpg}  & \includegraphics[width=0.19\textwidth]{figures/supplement/vis/dancetrack/0058/demo/walker/000124.jpg} & \includegraphics[width=0.19\textwidth]{figures/supplement/vis/dancetrack/0058/demo/walker/000128.jpg} & \includegraphics[width=0.19\textwidth]{figures/supplement/vis/dancetrack/0058/demo/walker/000132.jpg} & \includegraphics[width=0.19\textwidth]{figures/supplement/vis/dancetrack/0058/demo/walker/000136.jpg} \\
\raisebox{+1.8\normalbaselineskip}[0pt][0pt]{\rotatebox[origin=c]{90}{Walker}}    & \includegraphics[width=0.19\textwidth]{figures/supplement/vis/dancetrack/0058/demo/walker_plus_plus/000120.jpg}  & \includegraphics[width=0.19\textwidth]{figures/supplement/vis/dancetrack/0058/demo/walker_plus_plus/000124.jpg} & \includegraphics[width=0.19\textwidth]{figures/supplement/vis/dancetrack/0058/demo/walker_plus_plus/000128.jpg} & \includegraphics[width=0.19\textwidth]{figures/supplement/vis/dancetrack/0058/demo/walker_plus_plus/000132.jpg} & \includegraphics[width=0.19\textwidth]{figures/supplement/vis/dancetrack/0058/demo/walker_plus_plus/000136.jpg}
\end{tabular}
  \caption{Tracking results on the sequence \textit{0058} of the DanceTrack validation set. We analyze 5 frames centered around the frame \#128 at time $\hat{t}$ and spaced by $k\mkern1.5mu{=}\mkern1.5mu\text{4/30}$ seconds. We compare the self-supervised trackers QDTrack-S~\cite{fischer2022qdtrack}, QD-Walker (ours), and Walker (ours). On each row, boxes of the same color correspond to the same tracking ID.}  \label{fig:vis_dancetrack_demo_0058}
  \vspace{-2em}
\end{figure*}

\begin{figure*}[]
\centering
\scriptsize
\setlength{\tabcolsep}{1pt}
\begin{tabular}{cccccc}
 & $t=\hat{t}-2k$ & $t=\hat{t}-k$  & $t=\hat{t}$  & $t=\hat{t}+k$  & $t=\hat{t}+2k$ \\
\raisebox{+1.9\normalbaselineskip}[0pt][0pt]{\rotatebox[origin=c]{90}{QDTrack-S}} & \includegraphics[width=0.19\textwidth]{figures/supplement/vis/dancetrack/0058/errors/static_qdtrack/000120.jpg} & \includegraphics[width=0.19\textwidth]{figures/supplement/vis/dancetrack/0058/errors/static_qdtrack/000124.jpg} & \includegraphics[width=0.19\textwidth]{figures/supplement/vis/dancetrack/0058/errors/static_qdtrack/000128.jpg} & \includegraphics[width=0.19\textwidth]{figures/supplement/vis/dancetrack/0058/errors/static_qdtrack/000132.jpg} & \includegraphics[width=0.19\textwidth]{figures/supplement/vis/dancetrack/0058/errors/static_qdtrack/000136.jpg} \\
\raisebox{+1.7\normalbaselineskip}[0pt][0pt]{\rotatebox[origin=c]{90}{QD-Walker}}    & \includegraphics[width=0.19\textwidth]{figures/supplement/vis/dancetrack/0058/errors/walker/000120.jpg}  & \includegraphics[width=0.19\textwidth]{figures/supplement/vis/dancetrack/0058/errors/walker/000124.jpg} & \includegraphics[width=0.19\textwidth]{figures/supplement/vis/dancetrack/0058/errors/walker/000128.jpg} & \includegraphics[width=0.19\textwidth]{figures/supplement/vis/dancetrack/0058/errors/walker/000132.jpg} & \includegraphics[width=0.19\textwidth]{figures/supplement/vis/dancetrack/0058/errors/walker/000136.jpg} \\
\raisebox{+1.8\normalbaselineskip}[0pt][0pt]{\rotatebox[origin=c]{90}{Walker}}    & \includegraphics[width=0.19\textwidth]{figures/supplement/vis/dancetrack/0058/errors/walker_plus_plus/000120.jpg}  & \includegraphics[width=0.19\textwidth]{figures/supplement/vis/dancetrack/0058/errors/walker_plus_plus/000124.jpg} & \includegraphics[width=0.19\textwidth]{figures/supplement/vis/dancetrack/0058/errors/walker_plus_plus/000128.jpg} & \includegraphics[width=0.19\textwidth]{figures/supplement/vis/dancetrack/0058/errors/walker_plus_plus/000132.jpg} & \includegraphics[width=0.19\textwidth]{figures/supplement/vis/dancetrack/0058/errors/walker_plus_plus/000136.jpg}
\end{tabular}
  \caption{ID switches on the sequence \textit{0058} of the DanceTrack validation set. We analyze 5 frames centered around the frame \#128 at time $\hat{t}$ and spaced by $k\mkern1.5mu{=}\mkern1.5mu\text{4/30}$ seconds. We compare the self-supervised trackers QDTrack-S~\cite{fischer2022qdtrack}, QD-Walker (ours), and Walker (ours). On each row, boxes colored in green are correctly tracked, while blue ones represent ID switches.}  \label{fig:vis_dancetrack_errors_0058}
\end{figure*}

\clearpage
\begin{figure*}[]
\centering
\scriptsize
\setlength{\tabcolsep}{1pt}
\begin{tabular}{cccccc}
 & $t=\hat{t}-2k$ & $t=\hat{t}-k$  & $t=\hat{t}$  & $t=\hat{t}+k$  & $t=\hat{t}+2k$ \\
\raisebox{+1.9\normalbaselineskip}[0pt][0pt]{\rotatebox[origin=c]{90}{QDTrack-S}} & \includegraphics[width=0.19\textwidth]{figures/supplement/vis/dancetrack/0077/demo/static_qdtrack/000212.jpg} & \includegraphics[width=0.19\textwidth]{figures/supplement/vis/dancetrack/0077/demo/static_qdtrack/000217.jpg} & \includegraphics[width=0.19\textwidth]{figures/supplement/vis/dancetrack/0077/demo/static_qdtrack/000222.jpg} & \includegraphics[width=0.19\textwidth]{figures/supplement/vis/dancetrack/0077/demo/static_qdtrack/000227.jpg} & \includegraphics[width=0.19\textwidth]{figures/supplement/vis/dancetrack/0077/demo/static_qdtrack/000232.jpg} \\
\raisebox{+1.7\normalbaselineskip}[0pt][0pt]{\rotatebox[origin=c]{90}{QD-Walker}}    & \includegraphics[width=0.19\textwidth]{figures/supplement/vis/dancetrack/0077/demo/walker/000212.jpg}  & \includegraphics[width=0.19\textwidth]{figures/supplement/vis/dancetrack/0077/demo/walker/000217.jpg} & \includegraphics[width=0.19\textwidth]{figures/supplement/vis/dancetrack/0077/demo/walker/000222.jpg} & \includegraphics[width=0.19\textwidth]{figures/supplement/vis/dancetrack/0077/demo/walker/000227.jpg} & \includegraphics[width=0.19\textwidth]{figures/supplement/vis/dancetrack/0077/demo/walker/000232.jpg} \\
\raisebox{+1.8\normalbaselineskip}[0pt][0pt]{\rotatebox[origin=c]{90}{Walker}}    & \includegraphics[width=0.19\textwidth]{figures/supplement/vis/dancetrack/0077/demo/walker_plus_plus/000212.jpg}  & \includegraphics[width=0.19\textwidth]{figures/supplement/vis/dancetrack/0077/demo/walker_plus_plus/000217.jpg} & \includegraphics[width=0.19\textwidth]{figures/supplement/vis/dancetrack/0077/demo/walker_plus_plus/000222.jpg} & \includegraphics[width=0.19\textwidth]{figures/supplement/vis/dancetrack/0077/demo/walker_plus_plus/000227.jpg} & \includegraphics[width=0.19\textwidth]{figures/supplement/vis/dancetrack/0077/demo/walker_plus_plus/000232.jpg}
\end{tabular}
  \caption{Tracking results on the sequence \textit{0077} of the DanceTrack validation set. We analyze 5 frames centered around the frame \#222 at time $\hat{t}$ and spaced by $k\mkern1.5mu{=}\mkern1.5mu\text{5/30}$ seconds. We compare the self-supervised trackers QDTrack-S~\cite{fischer2022qdtrack}, QD-Walker (ours), and Walker (ours). On each row, boxes of the same color correspond to the same tracking ID.}  \label{fig:vis_dancetrack_demo_0077}
  \vspace{-2em}
\end{figure*}

\begin{figure*}[]
\centering
\scriptsize
\setlength{\tabcolsep}{1pt}
\begin{tabular}{cccccc}
 & $t=\hat{t}-2k$ & $t=\hat{t}-k$  & $t=\hat{t}$  & $t=\hat{t}+k$  & $t=\hat{t}+2k$ \\
\raisebox{+1.9\normalbaselineskip}[0pt][0pt]{\rotatebox[origin=c]{90}{QDTrack-S}} & \includegraphics[width=0.19\textwidth]{figures/supplement/vis/dancetrack/0077/errors/static_qdtrack/000212.jpg} & \includegraphics[width=0.19\textwidth]{figures/supplement/vis/dancetrack/0077/errors/static_qdtrack/000217.jpg} & \includegraphics[width=0.19\textwidth]{figures/supplement/vis/dancetrack/0077/errors/static_qdtrack/000222.jpg} & \includegraphics[width=0.19\textwidth]{figures/supplement/vis/dancetrack/0077/errors/static_qdtrack/000227.jpg} & \includegraphics[width=0.19\textwidth]{figures/supplement/vis/dancetrack/0077/errors/static_qdtrack/000232.jpg} \\
\raisebox{+1.7\normalbaselineskip}[0pt][0pt]{\rotatebox[origin=c]{90}{QD-Walker}}    & \includegraphics[width=0.19\textwidth]{figures/supplement/vis/dancetrack/0077/errors/walker/000212.jpg}  & \includegraphics[width=0.19\textwidth]{figures/supplement/vis/dancetrack/0077/errors/walker/000217.jpg} & \includegraphics[width=0.19\textwidth]{figures/supplement/vis/dancetrack/0077/errors/walker/000222.jpg} & \includegraphics[width=0.19\textwidth]{figures/supplement/vis/dancetrack/0077/errors/walker/000227.jpg} & \includegraphics[width=0.19\textwidth]{figures/supplement/vis/dancetrack/0077/errors/walker/000232.jpg} \\
\raisebox{+1.8\normalbaselineskip}[0pt][0pt]{\rotatebox[origin=c]{90}{Walker}}    & \includegraphics[width=0.19\textwidth]{figures/supplement/vis/dancetrack/0077/errors/walker_plus_plus/000212.jpg}  & \includegraphics[width=0.19\textwidth]{figures/supplement/vis/dancetrack/0077/errors/walker_plus_plus/000217.jpg} & \includegraphics[width=0.19\textwidth]{figures/supplement/vis/dancetrack/0077/errors/walker_plus_plus/000222.jpg} & \includegraphics[width=0.19\textwidth]{figures/supplement/vis/dancetrack/0077/errors/walker_plus_plus/000227.jpg} & \includegraphics[width=0.19\textwidth]{figures/supplement/vis/dancetrack/0077/errors/walker_plus_plus/000232.jpg}
\end{tabular}
  \caption{ID switches on the sequence \textit{0077} of the DanceTrack validation set. We analyze 5 frames centered around the frame \#222 at time $\hat{t}$ and spaced by $k\mkern1.5mu{=}\mkern1.5mu\text{5/30}$ seconds. We compare the self-supervised trackers QDTrack-S~\cite{fischer2022qdtrack}, QD-Walker (ours), and Walker (ours). On each row, boxes colored in green are correctly tracked, while blue ones represent ID switches.}  \label{fig:vis_dancetrack_errors_0077}
\end{figure*}

\clearpage
\begin{figure*}[]
\centering
\scriptsize
\setlength{\tabcolsep}{1pt}
\begin{tabular}{cccccc}
 & $t=\hat{t}-2k$ & $t=\hat{t}-k$  & $t=\hat{t}$  & $t=\hat{t}+k$  & $t=\hat{t}+2k$ \\
\raisebox{+1.9\normalbaselineskip}[0pt][0pt]{\rotatebox[origin=c]{90}{QDTrack-S}} & \includegraphics[width=0.19\textwidth,trim={10.0cm 10.0cm 10.0cm 1.0cm},clip]{figures/supplement/vis/dancetrack/0081/demo/static_qdtrack/000034.jpg} & \includegraphics[width=0.19\textwidth,trim={10.0cm 10.0cm 10.0cm 1.0cm},clip]{figures/supplement/vis/dancetrack/0081/demo/static_qdtrack/000037.jpg} & \includegraphics[width=0.19\textwidth,trim={10.0cm 10.0cm 10.0cm 1.0cm},clip]{figures/supplement/vis/dancetrack/0081/demo/static_qdtrack/000040.jpg} & \includegraphics[width=0.19\textwidth,trim={10.0cm 10.0cm 10.0cm 1.0cm},clip]{figures/supplement/vis/dancetrack/0081/demo/static_qdtrack/000043.jpg} & \includegraphics[width=0.19\textwidth,trim={10.0cm 10.0cm 10.0cm 1.0cm},clip]{figures/supplement/vis/dancetrack/0081/demo/static_qdtrack/000046.jpg} \\
\raisebox{+1.7\normalbaselineskip}[0pt][0pt]{\rotatebox[origin=c]{90}{QD-Walker}}    & \includegraphics[width=0.19\textwidth,trim={10.0cm 10.0cm 10.0cm 1.0cm},clip]{figures/supplement/vis/dancetrack/0081/demo/walker/000034.jpg}  & \includegraphics[width=0.19\textwidth,trim={10.0cm 10.0cm 10.0cm 1.0cm},clip]{figures/supplement/vis/dancetrack/0081/demo/walker/000037.jpg} & \includegraphics[width=0.19\textwidth,trim={10.0cm 10.0cm 10.0cm 1.0cm},clip]{figures/supplement/vis/dancetrack/0081/demo/walker/000040.jpg} & \includegraphics[width=0.19\textwidth,trim={10.0cm 10.0cm 10.0cm 1.0cm},clip]{figures/supplement/vis/dancetrack/0081/demo/walker/000043.jpg} & \includegraphics[width=0.19\textwidth,trim={10.0cm 10.0cm 10.0cm 1.0cm},clip]{figures/supplement/vis/dancetrack/0081/demo/walker/000046.jpg} \\
\raisebox{+1.8\normalbaselineskip}[0pt][0pt]{\rotatebox[origin=c]{90}{Walker}}    & \includegraphics[width=0.19\textwidth,trim={10.0cm 10.0cm 10.0cm 1.0cm},clip]{figures/supplement/vis/dancetrack/0081/demo/walker_plus_plus/000034.jpg}  & \includegraphics[width=0.19\textwidth,trim={10.0cm 10.0cm 10.0cm 1.0cm},clip]{figures/supplement/vis/dancetrack/0081/demo/walker_plus_plus/000037.jpg} & \includegraphics[width=0.19\textwidth,trim={10.0cm 10.0cm 10.0cm 1.0cm},clip]{figures/supplement/vis/dancetrack/0081/demo/walker_plus_plus/000040.jpg} & \includegraphics[width=0.19\textwidth,trim={10.0cm 10.0cm 10.0cm 1.0cm},clip]{figures/supplement/vis/dancetrack/0081/demo/walker_plus_plus/000043.jpg} & \includegraphics[width=0.19\textwidth,trim={10.0cm 10.0cm 10.0cm 1.0cm},clip]{figures/supplement/vis/dancetrack/0081/demo/walker_plus_plus/000046.jpg}
\end{tabular}
  \caption{Tracking results on the sequence \textit{0081} of the DanceTrack validation set. We analyze 5 frames centered around the frame \#40 at time $\hat{t}$ and spaced by $k\mkern1.5mu{=}\mkern1.5mu\text{3/30}$ seconds. We compare the self-supervised trackers QDTrack-S~\cite{fischer2022qdtrack}, QD-Walker (ours), and Walker (ours). On each row, boxes of the same color correspond to the same tracking ID. For ease of visualization, we crop all frames around an area of interest.}  \label{fig:vis_dancetrack_demo_0081}
  \vspace{-3em}
\end{figure*}

\begin{figure*}[]
\centering
\scriptsize
\setlength{\tabcolsep}{1pt}
\begin{tabular}{cccccc}
 & $t=\hat{t}-2k$ & $t=\hat{t}-k$  & $t=\hat{t}$  & $t=\hat{t}+k$  & $t=\hat{t}+2k$ \\
\raisebox{+1.9\normalbaselineskip}[0pt][0pt]{\rotatebox[origin=c]{90}{QDTrack-S}} & \includegraphics[width=0.19\textwidth,trim={10.0cm 10.0cm 10.0cm 1.0cm},clip]{figures/supplement/vis/dancetrack/0081/errors/static_qdtrack/000034.jpg} & \includegraphics[width=0.19\textwidth,trim={10.0cm 10.0cm 10.0cm 1.0cm},clip]{figures/supplement/vis/dancetrack/0081/errors/static_qdtrack/000037.jpg} & \includegraphics[width=0.19\textwidth,trim={10.0cm 10.0cm 10.0cm 1.0cm},clip]{figures/supplement/vis/dancetrack/0081/errors/static_qdtrack/000040.jpg} & \includegraphics[width=0.19\textwidth,trim={10.0cm 10.0cm 10.0cm 1.0cm},clip]{figures/supplement/vis/dancetrack/0081/errors/static_qdtrack/000043.jpg} & \includegraphics[width=0.19\textwidth,trim={10.0cm 10.0cm 10.0cm 1.0cm},clip]{figures/supplement/vis/dancetrack/0081/errors/static_qdtrack/000046.jpg} \\
\raisebox{+1.7\normalbaselineskip}[0pt][0pt]{\rotatebox[origin=c]{90}{QD-Walker}}    & \includegraphics[width=0.19\textwidth,trim={10.0cm 10.0cm 10.0cm 1.0cm},clip]{figures/supplement/vis/dancetrack/0081/errors/walker/000034.jpg}  & \includegraphics[width=0.19\textwidth,trim={10.0cm 10.0cm 10.0cm 1.0cm},clip]{figures/supplement/vis/dancetrack/0081/errors/walker/000037.jpg} & \includegraphics[width=0.19\textwidth,trim={10.0cm 10.0cm 10.0cm 1.0cm},clip]{figures/supplement/vis/dancetrack/0081/errors/walker/000040.jpg} & \includegraphics[width=0.19\textwidth,trim={10.0cm 10.0cm 10.0cm 1.0cm},clip]{figures/supplement/vis/dancetrack/0081/errors/walker/000043.jpg} & \includegraphics[width=0.19\textwidth,trim={10.0cm 10.0cm 10.0cm 1.0cm},clip]{figures/supplement/vis/dancetrack/0081/errors/walker/000046.jpg} \\
\raisebox{+1.8\normalbaselineskip}[0pt][0pt]{\rotatebox[origin=c]{90}{Walker}}    & \includegraphics[width=0.19\textwidth,trim={10.0cm 10.0cm 10.0cm 1.0cm},clip]{figures/supplement/vis/dancetrack/0081/errors/walker_plus_plus/000034.jpg}  & \includegraphics[width=0.19\textwidth,trim={10.0cm 10.0cm 10.0cm 1.0cm},clip]{figures/supplement/vis/dancetrack/0081/errors/walker_plus_plus/000037.jpg} & \includegraphics[width=0.19\textwidth,trim={10.0cm 10.0cm 10.0cm 1.0cm},clip]{figures/supplement/vis/dancetrack/0081/errors/walker_plus_plus/000040.jpg} & \includegraphics[width=0.19\textwidth,trim={10.0cm 10.0cm 10.0cm 1.0cm},clip]{figures/supplement/vis/dancetrack/0081/errors/walker_plus_plus/000043.jpg} & \includegraphics[width=0.19\textwidth,trim={10.0cm 10.0cm 10.0cm 1.0cm},clip]{figures/supplement/vis/dancetrack/0081/errors/walker_plus_plus/000046.jpg}
\end{tabular}
  \caption{ID switches on the sequence \textit{0081} of the DanceTrack validation set. We analyze 5 frames centered around the frame \#40 at time $\hat{t}$ and spaced by $k\mkern1.5mu{=}\mkern1.5mu\text{3/30}$ seconds. We compare the self-supervised trackers QDTrack-S~\cite{fischer2022qdtrack}, QD-Walker (ours), and Walker (ours). On each row, boxes colored in green are correctly tracked, while blue ones represent ID switches. For ease of visualization, we crop around an area of interest.}  \label{fig:vis_dancetrack_errors_0081}
\end{figure*}

\immediate\closein\imgstream

\end{document}